\documentclass{article}

\usepackage[final]{neurips_2020}

\usepackage{graphicx}
\usepackage{amsmath}
\usepackage{amssymb}
\usepackage{amsthm}
\usepackage{bbm}
\usepackage{color}
\usepackage[utf8]{inputenc} 
\usepackage[T1]{fontenc}    
\usepackage{booktabs}       
\usepackage{amsfonts}       
\usepackage{nicefrac}       
\usepackage{microtype}      
\usepackage[utf8]{inputenc}
\usepackage[english]{babel}
\usepackage{algorithm} 
\usepackage{algorithmic}
\usepackage[algo2e]{algorithm2e} 

\newtheorem{theorem}{Theorem}[section]

\newtheorem{lemma}[theorem]{Lemma}

\newtheorem{proposition}[theorem]{Proposition}
\newtheorem{definition}[theorem]{Definition}

\newtheorem{condition}[theorem]{Condition}
\newtheorem{assumption}[theorem]{Assumption}

\usepackage{amsmath,amsfonts,bm}

\def\1{\bm{1}}

\def\vzero{{\bm{0}}}
\def\vone{{\bm{1}}}

\DeclareMathAlphabet{\mathsfit}{\encodingdefault}{\sfdefault}{m}{sl}
\SetMathAlphabet{\mathsfit}{bold}{\encodingdefault}{\sfdefault}{bx}{n}

\newcommand{\R}{\mathbb{R}}

\newcommand{\norm}[1]{\left\|#1\right\|}

\newcommand{\inprod}[1]{\left\langle #1 \right\rangle}
\def\ep{\varepsilon}

\newcommand{\ed}[1]{\textcolor{blue}{[ED: #1]}}
\newcommand{\bitem}{\begin{itemize}}
\newcommand{\eitem}{\end{itemize}}
\newcommand{\benum}{\begin{enumerate}}
\newcommand{\eenum}{\end{enumerate}}
\newcommand{\beq}{\begin{equation}}
\newcommand{\eeq}{\end{equation}}
\newcommand{\beqs}{\begin{equation*}}
\newcommand{\eeqs}{\end{equation*}}
\newcommand{\bals}{\begin{align*}}
\newcommand{\eals}{\end{align*}}

\newcommand{\Pw}{\mathcal{P}_{t}}
\newcommand{\Pu}{P_{u_t}}
\usepackage{comment}
\newcommand{\ourAlgo}{rPGD }

\graphicspath{
{/figs/}
{/figs1/}
}
\title{Implicit Regularization and Convergence for \\Weight Normalization}
\author{%
Xiaoxia Wu\thanks{Equal Contribution, xwu@math.utexas.edu, dobriban@wharton.upenn.edu, tongzheng@utexas.edu, shanshanw@google.com}
\\
 \texttt{\small University of Texas at Austin}    
  \And
  Edgar Dobriban$^*$\\
  \texttt{\small University of Pennsylvania}  
  \And
   Tongzhenng Ren$^*$
    \\
 \texttt{ \small University of Texas at Austin}  
    \And
   Shanshan Wu$^*$ \\
    \texttt{\small Google Research}  
       \And
Zhiyuan Li\\
    \texttt{\small Princeton University}  
\And
Suriya Gunasekar\\
    \texttt{\small Microsoft Research} 
\And
Rachel Ward
    \\
 \texttt{\small University of Texas at Austin}    
\And
Qiang Liu
    \\
 \texttt{\small University of Texas at Austin}    
 }
 
\begin{document}

\maketitle
\begin{abstract}

Normalization methods such as batch \citep{ioffe2015batch}, weight \citep{salimans2016weight}, instance \citep{ulyanov2016instance}, and layer normalization \citep{ba2016layer} have been widely used in modern machine learning. 
Here, we study the weight normalization (WN) method \citep{salimans2016weight} and a variant called reparametrized projected gradient descent (rPGD) for overparametrized least squares regression. WN and rPGD reparametrize the weights with a scale $g$ and a unit vector $w$ and thus the objective function becomes \emph{non-convex}. We show that this non-convex formulation has beneficial regularization effects compared to gradient descent on the original objective. These methods adaptively regularize the weights and converge close to the minimum $\ell_2$ norm solution, even for initializations far from zero. For certain stepsizes of $g$ and $w$, we show that they can converge close to the minimum norm solution. This is different from the behavior of gradient descent, which converges to the minimum norm solution only when started at a point in the range space of the feature matrix, and is thus more sensitive to initialization.

\end{abstract}
\section{Introduction}

Modern machine learning models often have more parameters than data points, allowing a fine-grained adaptation to the data, but also suffering from the risk of over-fitting. To alleviate this, various explicit and implicit regularization methods are used. For instance, weight decay can control the model complexity by shrinking the norm of the weights, and dropout can reduce the model capacity by sub-sampling features during training \citep{gal2016dropout,mianjy2018implicit,arora2020dropout}. 
Recent state-of-the-art techniques such as batch, weight, and layer normalization \citep{ioffe2015batch, salimans2016weight,ba2016layer}, empirically have a regularization effect, e.g., as described in \cite{ioffe2015batch}, "batch normalisation acts as a regularizer, in some cases eliminating the need for dropout". 

While normalization methods are practically popular and successful, their theoretical understanding has only started to emerge recently. 
For instance, normalization methods make learning more robust to hyperparameters such as the learning rate \citep{wu2018wngrad,arora2019implicit}. Moreover, it has been argued that normalization methods can make the model robust to the shift and scaling of the inputs, preventing ``internal covariate shift" \citep{ioffe2015batch} as well as smooth or modify \citep{santurkar2018does,pmlr-v89-lian19a} the optimization landscape. 
\begin{figure}[tb]
\begin{minipage}[c]{0.57\textwidth}
    \caption{Comparison of the outputs $||\widehat{x}||=\|\widehat{g}\widehat{w}\|$ provided by GD, WN and \ourAlgo on an overparametrized linear regression problem (see Section \ref{sec:setup}). All algorithms (with stepsizes $0.005$) start from the same initialization and stop when the loss reaches $10^{-5}$. Note that the orange (rPGD) and  green (WN) curves overlap (see Lemma \ref{cor:flow} for explanation and Section~\ref{sec:lr_exp} for experimental details). GD converges to the minimum $\ell_2$-norm solution only when $\|x_0\|=0$, while WN and \ourAlgo converge close to the minimum norm solution for a wider range of initializations with smaller standard deviation. }
    \label{fig:lr_init_intro}
    \end{minipage} 
    \begin{minipage}[c]{0.4\textwidth}    \includegraphics[width=1.\textwidth]{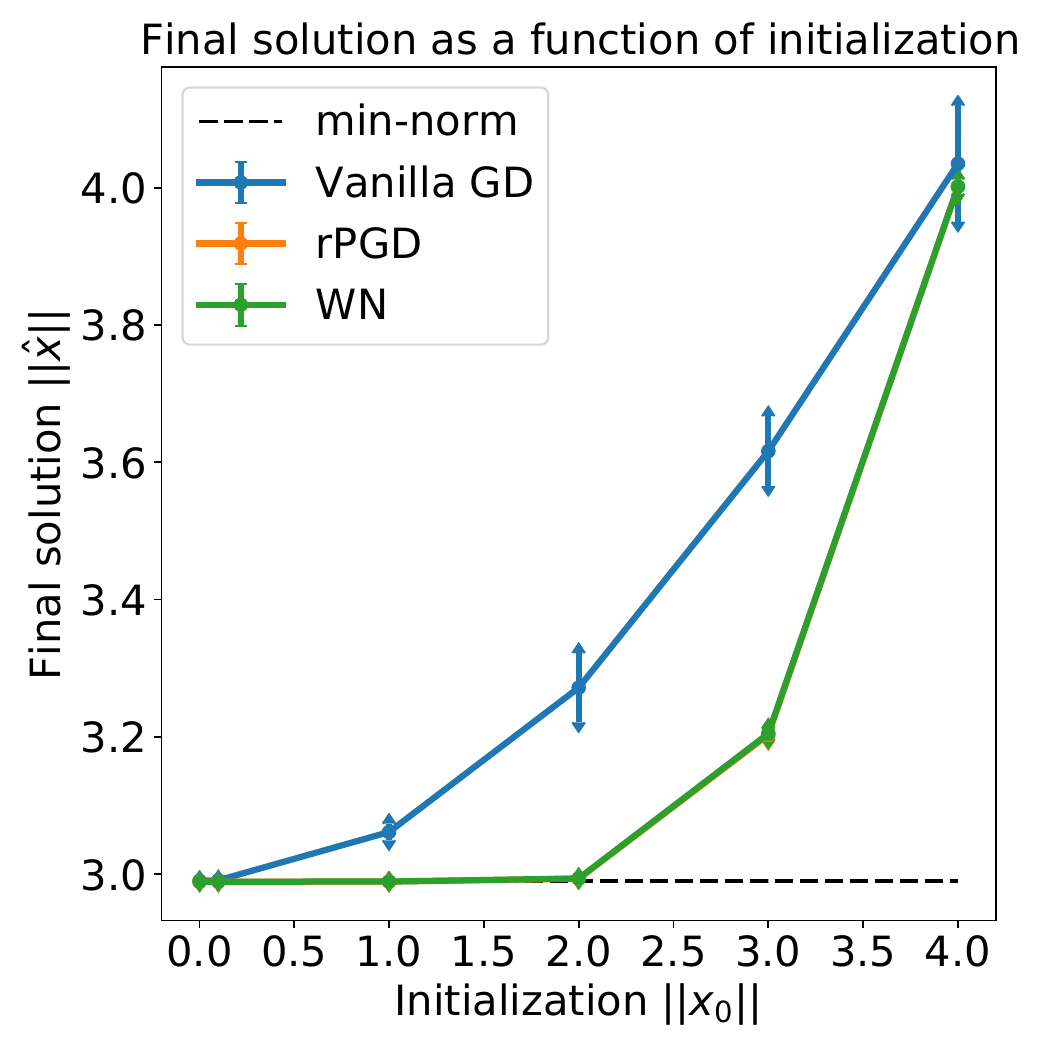}
    \end{minipage}
    \vspace{-0.4cm}
\end{figure}

Yet, a \emph{precise} characterization of the regularization effect of normalization methods in overparametrized models is not available. For  overparametrized models, there are typically infinitely many global minima, as shown e.g., in matrix completion \citep{ge2016matrix} and neural networks \citep{ge2017no}. Thus, we can analyze how different algorithms converge to different global minima as a way of quantifying implicit bias.  It is critical for the algorithm to converge to a solution with good generalization properties, e.g., \cite{zhang2016understanding,neyshabur2018the}, etc. For the key model of over-parameterized linear least squares, it is well-known that gradient descent (GD) converges to the minimum Euclidean norm solution when started from zero,  \citep[see e.g.][]{hastie2019surprises}. It has been argued that this may have favorable generalization properties in learning theory (norms can control the Radamechar complexity), as well as more recent analyses~\citep{BLLT19,hastie2019surprises, belkin2019two,liang2018just}.

However, for non-convex optimization, starting from the origin might be problematic -- this is true in particular in neural networks with ReLU activation function which is often used \citep{lecun2015deep}.  In neural networks, we often instead apply random initialization \citep{glorot2010understanding,he2015delving} which can for instance help escape saddle points \citep{lee2016gradient}.  Thus, it is important to study algorithms with initializations not close to zero.

With this in mind, we study how a particular normalization method, weight normalization (WN) \citep{salimans2016weight}, affects the choice of global minimum in overparametrized least squares regression.
WN writes the model parameters $x$ as $x=gw/\|w\|_2$, and optimizes over the "length" $g \in \R$ and the unnormalized direction $\smash{w\in \mathbb{R}^d}$ separately. Inspired by weight normalization, we also study a related method where we parametrize the weight as $x=gw$, with $g\in \R$ and a normalized direction $w$ with $\|w\|_2=1$, \citep[see e.g.][]{douglas2000gradient}. Different from WN, this method performs projected GD  on the unit norm vector $v$, while WN does GD on $w$ such that $w/\|w\|$ is the unit vector.  We call this variant the \emph{reparametrized projected gradient descent} (rPGD) algorithm. We show that the two algorithms (rPGD and WN) have the same limit  when the stepsize tends to zero. Arguing in both discrete and continuous time, we show that both find global minima \emph{robust to initialization}.

{\bf Our Contributions.}  
We consider the overparametrized least squares (LS) optimization problem, which is convex but has infinitely many global minima. 
As a simplified companion of WN in LS, we introduce the \ourAlgo algorithm (Alg. \ref{alg:main}), which is projected gradient descent on a nonconvex reparametrization of LS. 
We show that WN and \ourAlgo  have the same limiting flow---the WN flow---in continuous time (Lemma \ref{cor:flow}). We characterize the stationary points of the loss, showing that the nonconvex reparametrization introduces some additional stationary points that are in general not global minima. However, we also show that the loss still decreases at a geometric rate, if we can control the scale parameter $g$.

How to control the scale parameter? Perhaps surprisingly, we show the delicate property that the scale and the orthogonal complement of the weight vector are linked through an invariant (Lemma \ref{lemma:wperp}). This allows us to show that the WN flow converges at a geometric rate in spite of the non-convexity of the reparameterized objective. We precisely characterize the solution, and when it is close to the min norm solution. 


In discrete-time, when the stepsize is not infinitely small, we first consider a simple setting where the feature matrix is orthogonal and characterize the behavior of \ourAlgo (Theorem \ref{thm:orthogonal1}). We show that by appropriately lowering the learning rate for the scale $g$, \ourAlgo converges to the minimum $\ell_2$ norm solution.  We give sharp iteration complexities  and upper bounds for the stepsize required for $g$. We extend the result to general data matrices $A$ (Theorem \ref{thm:non-orthogonal}), where the results become more challenging to prove and a bit harder to parse.  This sheds light the empirical observation that only optimizing the direction $w$ training the last layer of neural nets improves generalization \citep{goyal2017accurate,xu2019understanding}.





\subsection{Setup}\label{sec:setup}
We use $\|\cdot\|$ for the $\ell_2$ norm, and consider the standard overparametrized linear regression problem:
\begin{equation}
    \min_{x\in \R^d}\frac{1}{2} \|Ax-y\|^2, \label{eqn:lsq}
\end{equation}
where $A\in \R^{m\times d}$ ($m<d$) is the feature matrix and $y\in\R^m$ is the target vector. 
Without loss of generality, we assume that the feature matrix $A$ has full rank $m$. 
This objective has infinitely many global minimizers, and among them let the minimum $\ell_2$-norm solution be $x^*$. Observe that $x^*$ is characterized by the two properties: (1) $Ax^*=y$; (2) $x^*$ is in the row space of the matrix $A$. We can describe condition (2) via Definition \ref{def:space}.
\begin{definition}
For any $z \in \mathbb{R}^d$, we can write $z=z^{\parallel}+z^{\perp}$ where $Az^{\parallel} = Az \text{ and } Az^{\perp} = 0.$ \label{def:space} 
\end{definition}
Then we can equivalently write condition (2) as $x^{*\parallel}=x^{*}.$
We focus on weight normalization and a related reparametrized projected gradient descent method. Notably, both transform the original convex LS problem to a non-convex problem, which increases the difficulty of theoretical analysis.

\paragraph{Weight normalization (WN)} 
WN reparametrizes the variable $x$ as $g\cdot w/\|w\|$, where $g\in\mathbb{R}$ and $w\in\mathbb{R}^d$, which leads to the following minimization problem:
\begin{align}
  \min_{g\in \mathbb{R}, w\in\mathbb{R}^d} h(w,g) = \frac12 \left\|g A w/\|w\|-y\right\|^2.   \label{eqn:main0}
\end{align}

We can write the min norm solution as $x^*=g^* w^*/\|w^*\|$, where $w^*$ is unique up to scale. However, we can always choose $w^*$ so that $g^*>0$, unless $x^*=0$, which implies that $y=0$. We exclude this degenerate case throughout the paper. The discrete time WN algorithm is shown in Algorithm \ref{alg:wn}. 

\begin{minipage}{0.47\textwidth}
\begin{small}
\begin{algorithm}[H]
    \caption{WN for \eqref{eqn:main0}}
    \label{alg:wn}
\begin{algorithmic}
    \STATE {\bfseries Input:} Unit norm $w_0$ and scalar $g_0$,iterations $T$, step-sizes $\{\gamma_t\}_{t=0}^{T-1}$ and $\{\eta_t\}_{t=0}^{T-1}$ \\
    \FOR{$t=0, 1, 2, \cdots,T-1$}
        \STATE $w_{t+1} = w_t - \eta_t  \nabla_{w}h(w_t, g_t)$
        \STATE $g_{t+1} = g_t - \gamma_t  \nabla_{g}h(w_t, g_t) $
        \STATE
    \ENDFOR
\end{algorithmic}
\end{algorithm}
\end{small}
\end{minipage}
\begin{minipage}{0.47\textwidth}
\begin{small}
\begin{algorithm}[H]
    \caption{\ourAlgo for \eqref{eqn:main}}
    \label{alg:main}
\begin{algorithmic}
    \STATE {\bfseries Input:} Unit norm  $w_0$ and $g_0$, number of iterations $T$, step-sizes $\{\gamma_t\}_{t=0}^{T-1}$ and $\{\eta_t\}_{t=0}^{T-1}$  
    \FOR{$t=0, 1, 2, \cdots,T-1$}
        \STATE $v_t = w_t - \eta_t \nabla_{w}f(w_t, g_t)$ (gradient step)
        \STATE $w_{t+1} = \frac{v_t}{\|v_t\|}$  \,\, (projection)
        \STATE $g_{t+1} = g_t - \gamma_t \nabla_{g} f(w_t, g_t)$ (gradient step)
    \ENDFOR
\end{algorithmic}
\end{algorithm}
\end{small}
\end{minipage}

\paragraph{Reparametrized Projected Gradient Descent (rPGD)}
Inspired by WN algorithm, we investigate an algorithm that directly updates the direction of $w$. See \cite{douglas2000gradient} for an example of such algorithms. 
Since the direction is a unit vector, we can perform projected gradient descent on it. To be more concrete, we reparametrize the variable $x$ as $gw$, where $g$ denotes the scale and $w\in\R^d$ with $\|w\|=1$ denotes the direction, and transform \eqref{eqn:main0} into the following problem: 
\begin{equation}
    \min_{g\in \mathbb{R}, w\in\mathbb{R}^d} f(w, g):= \frac{1}{2}\|Agw - y\|^2,\quad \textnormal{s.t.}\quad \|w\|=1.
    \label{eqn:main}
\end{equation}
The minimum norm solution can be uniquely written as $x^*=g^*w^*$, where $g^*>0$ and $\|w^*\|=1$. To solve \eqref{eqn:main}, we update $g$ with standard gradient descent, and update $w$ via projected gradient descent (PGD) (see Algorithm~\ref{alg:main}). We call this algorithm reparameterized PGD (rPGD).



One may observe that both algorithms can heuristically be viewed as a variation of  adaptive $\ell_2$ regularization, where the magnitude of the regularization depends on the current iteration. We refer the readers to Appendix \ref{sec:adareg} for a detailed discussion.


\section{Continuous Time Analysis} \label{sec:main}
%

In this section, we study the properties of a continuous limit of WN and rPGD, to give insight into the implicit regularization of normalization methods. We use constant stepsizes for both the update of the scale $g$ and weight $w$, and take them to zero in a way that their ratio remains a constant.
\begin{condition}[Stepsizes]\label{cond:stepsize}
For both Algorithms 1 and 2, use constant stepsizes $\eta_t = \eta$ and $\gamma_t=c\eta$ for $g$ and $w$ respectively, with $c\ge 0$ a fixed constant ratio. 
We take the continuous limit $\eta \to 0$.
\end{condition}
Setting $c=0$ amounts to fixing $g$ and only updating $w$.
We first prove that the continuous limit of the dynamics of $(g_t, w_t/\|w_t\|)$ for WN evolves the same as the continuous limit of the dynamics of $(g_t,w_t)$ for rPGD, assuming we start with $\|w_0\|=1$ for WN. The proof can be found in Appendix \ref{pf:flow}.
\begin{lemma}[Limiting flow for WN and rPGD]\label{cor:flow}
Assume Condition \ref{cond:stepsize} and that $\|w_0\|=1$ for WN. Then WN (Algorithm \ref{alg:wn}) with $(g_t, w_t/\|w_t\|)$ and  rPGD (Algorithm  \ref{alg:main}) with $(g_t, w_t)$  have the same limiting dynamics, which we call \textbf{WN flow}. This is given by the pair of ordinary differential equations
\begin{align}
\frac{dg_t}{dt} = - c\nabla_g f(w_t,g_t)  & \quad
\frac{dw_t}{dt} = - g_t\Pw\left(\nabla_w f(w_t,g_t)\right).
\end{align}
Here $f$ is from \eqref{eqn:main}. With $r = y - A g w$ to denote the residual, $\nabla_w f= A^\top  r$,  $\nabla_g f=w^\top A^\top r$, and $\Pw = I-w_tw_t^\top/\|w_t\|^2$ the projection matrix onto the space orthogonal to $w_t$.

\end{lemma}
While the flow is valuable, the nonconvex reparametrization introduces some new stationary points. We characterize them, and later use this to understand the convergence.
\begin{lemma}[Stationary points] \label{lem:stat_p_loss} Suppose the smallest eigenvalue of $AA^\top $, is positive, $\lambda_{\min}:=\lambda_{\min}(AA^\top )>0$.
The stationary points of the reparameterized loss from \eqref{eqn:main0} either (a) have loss equal to zero, or (b) belong to the set  $\mathcal{S}:= \{(g,w): g =0, y^\top  Aw = 0\}.$ If the loss \eqref{eqn:main0} at $g,w$ is strictly less than the loss at $(g=0,w)$, i.e. $\|y\|^2> \|Agw-y\|^2$, we are always in case (a).
\end{lemma}
It is a folklore result that under gradient flow, the loss is non-increasing even in the nonconvex case \citep[see e.g.][]{rockafellar2009variational}.
For the WN gradient flow, we can make this folklore rigorous and, provided the scale parameter $g_t$ is lower bounded, show that the loss decreases at a \emph{geometric rate}.
\begin{lemma}[Rate of $\|r_t\|$] \label{lem:loss}
Under the setting of Lemma \ref{cor:flow}, we have the bounds:
\begin{align}\label{ulb}
 -\max\{g_t^2, c\}\|A^\top r_t\|^2\leq d[1/2\|r_t\|^2]/dt \leq -\min\{g_t^2, c\}\|A^\top r_t\|^2 \le 0.
\end{align}
This shows that $\|r_t\|$ is non-increasing. If for some $C>0$, $g_t>C$ for all $t$, then the loss decreases geometrically at rate $\min(C^2,c)$.
\end{lemma}
How can we control the scale parameter? Perhaps surprisingly, we show that the scale parameter and the orthogonal complement of the weight vector are linked through an \emph{invariant}.




\begin{lemma} [Invariant] Assume $c>0$ in Condition \ref{cond:stepsize}. Under the setting from Lemma \ref{cor:flow}, let $\smash{w_t = w_t^\perp+w_t^{\parallel}}$ as defined in Definition \ref{def:space}. We have at time $t>0$,
\begin{align}
w_t^\perp= \exp\left({\frac{g_0^2-g_t^2}{2c}}\right) w_0^\perp \quad \text{and so}\quad \|w_t^\perp\|^2 \cdot  \exp(g_t^2/2c)
=
\|w_0^\perp\|^2 \cdot  \exp(g_0^2/2c).
\label{eq:inv1}
\end{align}\label{lemma:wperp}
\end{lemma}
Lemma \ref{lemma:wperp} shows that the orthogonal complement $w_t^\perp$ can change during the WN flow dynamics. This is the key property of WN that can yield additional regularization. Lemma \ref{lemma:wperp} also implies that $\|w_t^\perp\|^2 \cdot  \exp(g_t^2/2c)$ is invariant along the path.
If we initialize with small $|g_0|$ and $|g_t|$ is greater than $|g_0|$ (we will describe the dynamics of $g_t$ in the next part), then $\|w_t^{\perp}\|^2$ will decrease, and we get close to the minimum norm solution. This is in contrast to gradient descent and flow, where $\|w^{\perp}\|^2$ is preserved (see e.g., ~\citep{hastie2019surprises}). 

The invariant \eqref{eq:inv1} in the optimization path holds for certain more general settings. Specifically, it holds for linearly parametrized loss functions that only depend on a small dimensional linear subspace of the parameter space (e.g., overparametrized logistic regression). See Appendix \ref{sec:beyond}. 
Equipped with the above lemmas, we can discuss the solution and implicit regularization effect of the WN flow.
\begin{theorem}[WN flow Solution]\label{thm:convergence}
Assume Condition \ref{cond:stepsize} and $\lambda_{\min}>0$. 
Suppose we initialize the WN flow at $g_0,w_0$, such that $\|w_0\|=1$. 
We have that either
(a) the loss converges to zero, or
(b) the iterates $(g_t,w_t)$ converge to a stationary point in $\mathcal{S}$ as defined in Lemma \ref{lem:stat_p_loss}. In case (a), we characterize the solutions based on $g_t$:
\begin{itemize}
    \item[\textbf{Part I.}] %
If $c>0$, and the loss converges to zero, the solution can be expressed as  \begin{align}\lim_{t\to \infty }g_{t}w_{t} = x^*+g^* w_{0}^{\perp}\exp\left(\frac{g_0^2-g^{*2}}{2c}\right).\label{eq:sol}
    \end{align}
   \item[\textbf{Part II.}] 
   If $c=0$ and $A$ is orthogonal, i.e., $AA^\top =I$, then $w_t\to w^*$. If $A$ is not orthogonal, then the flow still converges to a point $\tilde w_0$ in the row space of $A$ (i.e,  $\tilde w_0^\perp=0$). When restarting the WN flow with $c>0$ from $g_0,\tilde w_0$, then $(g_0,\tilde w_0)\to (g^*,w^*)$.
\label{thm:rPGD_flow_with_fixg}
\end{itemize} 
\end{theorem}
We defer the proofs of Lemmas \ref{lem:loss}, \ref{lemma:wperp} and Theorem \ref{thm:convergence} to Appendix \ref{pf:rpgdf}. Part I of Theorem \ref{thm:convergence} shows that, if we initialize with $g_0^2\leq g^{*2}$ and we are not stuck at $\mathcal{S}$, the WN flow will converge to a solution that is close to the minimum norm solution. Compared with GD where the final solution is $x_t = x^*+ g^*w_0^{\perp}$, WN flow has smaller component in the orthocomplement of the row space of $A$. In contrast, if $g_0^2> g^{*2}$, then WN flow can converge to a solution that is \emph{farther} from $x^*$ than GD. 


{Part II} in Theorem \ref{thm:convergence} shows a distinction between orthogonal and general $A$. For orthogonal $A$, even fixing the scale $g_0$ we can converge to the direction of the minimum norm solution. Although we do not directly recover $g^*$ in the flow, this can be recovered as $|g^*| = \|y\|$. For general $A$ with fixed $g$, we do not necessarily converge to the right direction, only to the row span of $A$. However, if we run the flow with $c = 0$ until convergence, and then turn on the flow for $g$ (i.e. set $c > 0$), we converge to the minimum norm solution. The results for discrete time presented later mirror this. See Figure \ref{fig:simple} for an illustration. We mention that the flow for the fixed $g$ case is well known \citep[See e.g.][Section 1.6]{helmke2012optimization}), in the special case that the matrix $A$ is square.  
Theorem \ref{thm:convergence} provides no rate of convergence. 
By our results on the rate of decay of $\|r_t\|$, and by controlling $g_t$ using the invariant, we can provide a convergence rate below. 


\begin{theorem}[Convergence Rate]\label{convergence_rate} Suppose that Condition \ref{cond:stepsize} with $c > 0$ holds, 
that $\|w_0\|=1$, 
and that the smallest eigenvalue $\lambda_{\min}$ of $AA^\top $ is strictly positive.  If
$g_0^2>2c\log(1/\|w_0^\perp\|)$, the loss along the WN flow path $(g_t,w_t)$ satisfies $f(w_T,g_T)\leq \ep$ after time $T\ge T_0$, where
\[T_0= \frac{\log(f(w_0,g_0)/\ep)}{\lambda_{\min} \min\left\{2c\log\|w_0^\perp\|+g_0^2,c\right\}}.\]
\end{theorem}


The theorem states that the loss converges geometrically as long as $g_0^2$ is above the required threshold $2c\log(1/\|w_0^\perp\|)$.
The theorem focuses on convergence, not implicit regularization. 
However, as described above, the regularization is favorable if $|g_0|<|g^*|$.




\begin{figure}[t]
\begin{minipage}[c]{0.4\textwidth}
        \centering
    \includegraphics[width=1.\linewidth]{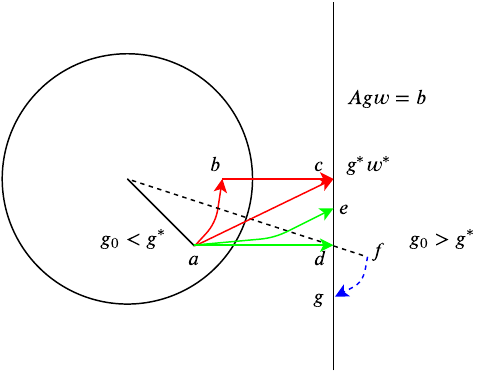}
\end{minipage}
\begin{minipage}[c]{0.57\textwidth}
    \caption{Consider the function $f(w_1, w_2, g)$ with $A = a \in \mathbb{R}^{1 \times 2}$.
    Then GD converges to $d$, while rPGD and WN could result in a point ($e$ or $c$) closer to minimum norm $c$ depending on the stepsize schedule of $g$. Part I in Theorem \ref{thm:convergence} suggests that  rPGD and WN  will follow the path  $a\to e$ if $\gamma_t$ and $\eta_t$ converge to zero  at the same rates, and Part II implies the red path $a\to b\to c$ to the minimum norm solution (if $g_0$ is fixed for a certain time, and updated later). The optimal path $a\to c$ is taken when $g$ is updated in a careful way. 
    On the other hand, starting with $g_0>g^*$, for instance at $f$, \eqref{eq:sol} shows the limit is $g$, further away from $g^*w^*$. } 
    \label{fig:simple}
    \end{minipage} 
\end{figure}

\textbf{A Concrete Example.} To gain more insight, we provide here a simple example  (see also Figure \ref{fig:simple}). Suppose we have a two-dimensional parameter $w$, and we make a 1-dimensional observation using the matrix $A=[1,0]$, and $y=1$. Then, the equation we are solving is $gw[1]=1$ (where square brackets index coordinates of vectors), and the minimum norm solution is $w=[1,0]^\top $, with $g^*=1$. Our results guarantee that the WN flow converges to either (1) a zero of the loss, or (2) to a stationary point such that $g=0$ and $w^\top  A^\top  y=0$. The second condition reduces to $w[1] \cdot y=0$. Now, if $y\neq 0$ (which is the typical case), then this reduces to $w[1]=0$, and since $\|w\|=1$, we have two solutions $w[2] = \pm 1$. So this leads to two spurious stationary points $(g,w) = (0,[0,\pm 1]^\top )$, which are not global minima. The loss value at these points is $1$, and so if we start at any point such that the loss is less than one, then we converge to a global mininum. If $y=0$, then this leads to infinitely many stationary points, i.e. all of those with $g=0$, but these turn out to be global minima.

Suppose moreover that we start with $w_0=[0,1]^\top $ and set $c=1$.  Suppose now that we start with some $g_0 \neq 0$. Then WN flow converges to a solution $gw = [1,\exp([g_0^2-1]/2)]^\top.$ If $g_0$ is relatively small, this quantity is close to $x^*=[1,0]^\top$, closer than the gradient flow solution $[1,1]$.
\section{Discrete Time Analysis}


\label{sec:orthogonal}
In this section, we switch to discrete time. It turns out that analyzing rPGD is more tractable than WN, so we will focus on rPGD. Since the two algorithms collapse to the same flow in continuous time, their dynamics should be ``close" in finite time, especially in the small stepsize regime. 
We show that rPGD with properly chosen learning rates converges close to the min norm solution even when the initialization is \emph{far away from the origin}. 
We study rPGD based on the intuition that $\|w_t^{\perp}\|$ decreases after the normalization step.

\paragraph{Orthogonal Data Matrix.} Consider first the simple case where the feature matrix $A$ has orthonormal rows, i.e., $AA^\top=I$. Our strategy for rPGD to reach the minimum $\ell_2$-norm solution is  \emph{to use  the optimal stepsize for $w$ and a small stepsize for $g$ such that  $g_t^2< g^{*2}$ for all iterations}. The key intuition is that with a small stepsize, the loss stays positive and ensures the direction $w_t$ has sufficient time to find $w^*$.  On the contrary, if we use a large stepsize for $g$, then it is possible for $g_t$ to be greater than $g^*$ so that $w_t$ can potentially converge to the wrong direction. 

\begin{condition}\label{con:stepsize1}(Two-stage learning rates)
We update $w$ with its optimal step-size  $\eta_t=1/g_t^2$.\footnote{The Hessian for $w$ in problem \ref{eqn:main} is $\nabla_{w}^2f(w,g) = g^2A^\top A$. For orthogonal $A$, $\lambda_{\max}(\nabla_{w}^2f(w,g)) =g^2$.} For the  stepsize of $g_t$, we use two constant values: (a) $\smash{\gamma_t=\gamma^{(1)}}$ when $0\leq t\leq T_1$; (b) $\smash{\gamma_t= \gamma^{(2)}}$  when  $t\geq T_1+1$, for a $T_1$ specified below. 
\end{condition}
\begin{theorem}[Convergence for Orthogonal Matrix $A$]\label{thm:ort}
Suppose the initialization satisfies $0< g_0 < g^*$, and that $w_0$ is a vector with $\|w_0\|=1$. Let $\delta_0 =(g^*)^2 -(g_0)^2$. 
Set an error parameter $\ep>0$ and
the stepsize given in Condition \ref{con:stepsize1} with a hyper-parameter $\rho\in(0,1]$ for $\gamma^{(1)}$. Running the rPGD algorithm, we can reach  $\|w_{T_1}^\perp\|\leq \ep$ and $g_{T_1}^2\leq g^{*2}-\rho \delta_0$ after $T_1$ iterations, and   $\|w_{T}^\perp\|\leq \ep$ and $\|Ag_{T}w_{T}-b \|^2\leq 3\ep g^{*2}$ after $T = T_1 + T_2$ iterations, if we set stepsizes as follows.
\begin{itemize}
    \item[(a)] Set $\gamma^{(1)} = \mathcal{O}\left(\frac{\rho }{ \log(1/\ep)} \left(\frac{g_0}{g^{*}}\right)^2\log\left((1-\rho)\frac{g^{*}}{g_0} +\rho \right)\right)$ and  $\gamma^{(2)}\leq \frac{1}{4}$. Then we have \[
T_1= \mathcal{O}\left(\frac{(g^*)^2}{\rho\delta_0} \log\left( \frac{1}{\ep}\right)\right); \quad T_2 = \mathcal{O}\left( \frac{1}{\gamma^{(2)}}\log\left(\frac{(\rho 
\delta_0/g^{*2})^2}{\ep }\right) \right). \]
    \item[(b)] Set $\gamma^{(1)}=0$ and
    $\gamma^{(2)}< \frac{1}{4}$. Then we have 
    \[ T_1 = \mathcal{O}\left(\frac{g_0^2}{\delta_0}\log\left(\frac{1}{\ep}\right)\right);\quad T_2 =\mathcal{O}\left(\frac{1}{\gamma^{(2)}}\log \left( \frac{\sqrt{\delta_0/g^{*2}}}{\ep}\right)\right). \]
\end{itemize}
\label{thm:orthogonal1}
\end{theorem}
We restate the theorem with the explicit forms of $T_1$ and $T_2$, along with the proof, in Appendix \ref{sec:proof_orthogonal} and \ref{sec:proof_orthogonal1}. The theorem  requires knowing $g^*$, which can be approximated by $\|y\|$ (as $g^*= \|y\|$). When all parameters other than $\ep$ are treated as constants, this shows that rPGD converges to the minimum norm solution with the same rate $\log(1/\ep)$ as standard GD starting from the origin. However, the constants in front the $\log(1/\ep)$ can be large: e.g., in case (a), $(g^*)^2/\rho\delta_0$ can be $\gg 1$ if $\rho$ is small or if $|g_0|/|g^*|\approx 1$.
This first $T_1$ iterations allow $w_t$ to "find" $w^*$, while the remaining $T_2$ allow $g_t$ to converge to $g^*$. Both cases show an intrinsic tradeoff between $T_1$ and $T_2$: a larger $\delta_0$ (being far from $g^*$) leads to faster convergence in the first phase for  $\smash{\|w_t^\perp\|}$ (i.e. smaller $T_1$), but slower convergence in $g_t$ and loss (i.e. larger $T_2$). Specifically, notice that $\delta_0$ is in the denominator of $T_1$ but in the numerator of $T_2$. 

Our proof shows that $g_t$ is always increasing for any $g_0>0$ (c.f. Lemma \ref{lem:g_increase}). Moreover, $\|w_t^{\perp}\|$ decreases at a geometric rate,
$\|w_{t+1}^{\perp}\|^2 \leq ({g_{t}^2}/{g^{*2}})\|w_t^{\perp}\|^2$, as long as $|g_t|$ is not too close to $|g^{*}|$ (c.f. Lemma \ref{lem:Aw}  or Equation \eqref{eq:wperp_decrease}). This is why the condition $g_{T_1}^2\leq g^{*2}-\rho \delta_0$ is needed, ensuring that $g_t$ is far away from $g^*$ in all steps before $T_1$. This is also why we require a stepsize for $g_t$ of order $1/\log(1/\ep)$ (c.f. Equation \ref{eq:control_g}), which is smaller than the constant stepsize in usual GD. Here $\rho$ leads to a tradeoff between $T_1$ and $T_2$: a smaller $\rho$ results in larger $T_1$ but smaller $T_2$, vice versa. When $\rho\approx \ep$, we  have $T_1=\mathcal{O}(1/\ep)$ up to log factors, (a slower rate) and $T_2=\mathcal{O}\left(1\right)$. \footnote{Note that the bound for $T_1$ could be tightened, possibly to $\log(1/\ep)$, by using refined analysis at the step from \eqref{eq:tighter1} to  \eqref{eq:wperp_decrease}.} A constant order $\rho$ leads to a faster $\log(1/\ep)$ rate. However, we choose to state the result for the entire range of $\rho\in(0,1]$ for completeness. When $\rho = 1$, the stepsize $\gamma^{(1)}$ becomes zero, hence $g_t$ does not change. In this case, we can get a stronger
result for $T_1$ (stated in (b)) using a slightly different method of proof, improving the bounds of case (a) respectively with a factor of $(g^*/g_0)^2>1$ for $T_1$.






We remark that for orthogonal $A$ with the optimal stepsize ($1/g_t^2$) for $w_t$ and $g_0\neq 0$, we have  $Aw_{t+1} = {Ag^* w^*}/{(g_t\|v_t\|)} \neq 0$ (c.f. Lemma \ref{lem:v_norm}).
Thus we can escape the saddle points and reach the global minimum, unlike in continuous time where we can be stuck at the stationary points $\mathcal{S}$. 

We reiterate that our motivation is not to outperform other methods (e.g. GD starts from zero) in search of the minimum norm solution, but to characterize the regularization effect of weight normalization and shed light on the empirical observation that fixing the scalar $g$ and only optimizing the directions $w$ in training the last layer of neural networks can improve generalization \citep{goyal2017accurate,xu2019understanding}. This is, to our knowledge, the first kind of theory on how to control the learning rates of parameters in weight normalization such as to converge to minimum norm solutions \emph{for initialization not close to origin}, which may have beneficial generalization properties.
\paragraph{General Data Matrix.} Inspired by the analysis for orthogonal $A$, we now study general data matrices. As we have seen from the orthogonal $A$ case, the stepsize for the scale parameter should be extremely small or even $0$ to make $\|w_t^{\perp}\|$ small. Thus, for simplicity, we focus on fixing $g:=g_0$ and update only $w$ using rPGD so that the orthogonal component $w^{\perp}$ decreases geometrically until $\|w_{T_1}^\perp\|\leq \ep$. In addition, we notice from the analysis in Theorem \ref{thm:ort} that updating $g_t$ and $w_t$ separately after $t>T_1$ (i.e., reaching small $\smash{\|w_t^\perp\|}$) shows no advantage over GD using $x=gw$. Thus, the best strategy to find $g^*w^*$ is to use rPGD only updating $w_t$ (so $g_t=g_0<g^*$) and then apply standard GD after $T_1$ once we have $\|w_{T_1}^\perp\|\leq \ep$. We focus on the complexity of $T_1$ in the remainder, as the remaining steps are standard GD, which is well understood.

Even though we fix $g$, the problem is still non-convex because the projection is on the \emph{sphere} (rather than the \emph{ball}), a non-convex surface. However, suppose we can ensure that after each update, the gradient step $v_t = w_t - \eta_t \nabla_w f(w_t,g_t)$ has norm $\|v_t\|\geq 1$.  Then the following two constrained non-convex problems are equivalent: 
\[ \min_{w\in\mathbb{R}^d} \|Ag_0w-y\|^2 \text{ s.t. } w\in\{w, \|w\|=1  \} \quad \Leftrightarrow\quad \min_{w\in\mathbb{R}^d} \|Ag_0w-y\|^2 \text{ s.t. } w\in\{w, \|w\|\leq1 \}\]
Thus our analysis will focus on showing that $\|v_t\|\ge1$. Note that, without loss of generality we can always scale $A$ so that its largest singular value is one.
\begin{proposition}[General Matrix $A$]
\label{thm:non-orthogonal}
Fix $\delta > 0$, and fix a full rank matrix $A$ with $\lambda_{\max}(AA^\top)=1$. With a fixed $g=g_0$ satisfying $g_0\leq[g^*\lambda_{\min}(AA^\top)]/(2+\delta)$, we can reach a solution with $\|w_{T_1}^\perp\|\leq \ep$ in a number of iterations $$T_1= \log\left( \frac{\|w_0^{\perp}\|}{\ep}\right)/\log(1+\delta).$$
\end{proposition}
The proof is in Appendix \ref{sec:general}. The proposition implies that if we set a small $g_0=\mathcal{O}(g^*\lambda_{\min}(AA^\top))$ for general $A$ and $w^*$, running rPGD with fixed $g_0$ helps regularize the iterates. After starting from $w_{T_1}$, we can converge close to the minimum norm solution using standard GD. If the eigenvalues of $A$ are "not too spread out", we can get a better condition for $g_0$ using concentation inequalities for eigenvalues. See inequality \eqref{eq:g02}  in Proposition \ref{thm:non-orthogonal1} for more details. 


\section{Discussion} 
\label{relw}
{\bf Limitation of our work.}
It is important to recognize the limitations of our work. First, our theoretical work only addresses weight normalization (not batch, layer, instance or other normalization methods), and only concerns the setting of linear least squares regression. While this may seem limiting, it is still significant: even in this setting, the problem is not understood, and leads to intriguing insights. In fact there is some recent work on Neural Tangent Kernels arguing overparametrized NNs can be equivalent to linear problems, see e.g., \citep{jacot2018neural,du2018gradient,lee2019wide}, etc. Second, the continuous limit is only an approximation; however it leads to elegant and interpretable results, which are moreover also reflected in simulations. Third, some of our results concern a two-stage algorithm where the scale is fixed for the first stage; nevertheless, our results on the standard ``one-stage" algorithm in continuous time suggest such discrete-time results extend to the situation where the scale is not fixed, but slowly-varying for the first stage.

{\bf Related Work.}
While there is a large literature on weight normalization and implicit regularization (see Sec \ref{relw}), our work differs in crucial ways. We study the overparametrized case and characterize the implicit regularization for a broad range of initializations (unlike works that study initialization with small norm). Also, we prove convergence and characterize the solution explicitly (unlike works such as \citep{gunasekar2018characterizing} that assume convergence to minimizers). Below we can only discuss a small number of related works.

\textit{Implicit regularization.} 
It has been recognized early that optimization algorithms can have an implicit regularization effect, both in applied mathematics \citep{strand1974theory}, and in deep learning \citep{morgan1990generalization,neyshabur2014search}. It has been argued that ``algorithmic regularization" can be one of the main differences between the perspectives of statistical data analysis and more traditional computer science \citep{mahoney2012approximate}. 

Theoretical work has shown that gradient descent is a form of regularization for exponential-type losses such as logistic regression, converging to the max-margin SVM for separable data \citep{soudry2018implicit, PBL19}, as well as for non-separable data \citep{ji2019implicit}. Similar results have been obtained for other optimization methods \citep{gunasekar2018characterizing}, as well as for matrix factorization \citep{gunasekar2017implicit,arora2019implicit}, sparse regression \citep{vavskevivcius2019implicit}, and connecting to ridge regression \citep{ali2018continuous}. For instance, \citep{li2018algorithmic} showed that GD with small initialization and small step size finds low-rank solutions for matrix sensing. There have also been arguments that neural networks perform a type of self-regularization, some connecting to random matrix theory \citep{martin2018implicit,mahoney2019traditional}. Popular methods for regularization include weight decay (a.k.a., ridge regression) \citep{dobriban2018high,liu2019ridge}, dropout \citep{wager2013dropout}, data augmentation \citep{chen2019invariance}, etc.

\textit{Convergence of normalization methods.} \citep{salimans2016weight} argued that their proposed weight normalization (WN) method, optimizing $x=gw/\|w\|_2$ over $g\ge 0$ and $w\in \mathbb{R}^d$, increases the norm of $w$, and leads to robustness to the choice of stepsize. \citep{hoffer2018norm} studied normalization with weight decay and learning-rate adjustments.  \citep{du2018gradient} proved that GD with WN from randomly initialized weights could recover the right parameters with constant probability in a one-hidden neural network with Gaussian input. \citep{pmlr-v97-ward19a} connected the WN with adaptive gradient methods and proved the sub-linear convergence for both GD and SGD.  \citep{cai2019quantitative} showed that for under-parametrized least squares regression (which is different from our over-parametrized setting), batch normalized GD converges for arbitrary learning rates for the weights, with linear convergence for constant learning rate. Similar results for scale-invariant parameters can be found in  \citep{arora2018theoretical} with more general models, extending to the non-convex case. \citep{pmlr-v89-kohler19a} proved linear convergence of batch normalization in halfspace learning and neural networks with Gaussian data, using however parameter-dependent learning rates and optimal update of the length $g$. \citep{luo2018towards} analyzed batch normalization by using a basic block of neural networks and concluded that batch normalization has implicit regularization. \citep{dukler2020optimization} discussed the convergence of two-layer ReLU network with weight normalization under the NTK regime. However, none of the above give the invariants we do.

\textit{Nonlinear Least Mean Squares (NLMS) .} 
Normalization methods are possibly related to the Nonlinear Least Mean Squares (NLMS) methods from signal processing \citep[see e.g.][]{proakis2001digital,haykin2002least, haykin2005adaptive,hayes2009statistical}. NLMS can be viewed as an online algorithm where the samples $a_t,y_t$ ($a_t$ are the rows of $A$, $y_t$ are the entries of $y$) arrive in an online fashion, and we update the iterates as $x_{t+1}=x_t - \eta r_t a_t/\|a_t\|^2$, where $r_t = y_t - x_t^\top a_t$ are the residuals. There is a connection to randomized Kaczmarcz methods \cite{strohmer2009randomized}. However, it is not obvious how they are related to weight normalization or rPGD/WN, e.g., these methods are under online setting, while rPGD/WN are offline. 
\section*{Broader Impact}

Our work is on the foundations and theory of machine learning. One of the distinctive characteristics of contemporary machine learning is that it relies on a large number of "ad hoc" techniques, that have been developed and validated through computational experiments. For instance, the optimization of neural networks is in general a highly nonconvex problem, and there is no complete theoretical understanding yet as to how exactly it works in practice. Moreover, there a large number of practical "hacks" that people have developed that help in practice, but lack a solid foundation. Our work is about one of these techniques, weight normalization. We develop some nontrivial theoretical results about it in a simplified "model". This work does not directly propose any new algorithms. But we hope that our work will have an impact in practice, namely that it will help practitioners understand what the WN method is doing (important, as people naturally want to understand and know "why" things work), and possibly in the future, help us develop better algorithms (here the principle being that "if you understand it you can improve it", which has been useful in engineering and computer science for decades). 
\section*{Acknowledgments}
The authors thank Nathan Srebro and Sanjeev Arora for constructive suggestions.
XW, ED, SG, and RW thank the Institute for Advanced Study for their hospitality during the Special Year on Optimization, Statistics, and Theoretical Machine Learning.  XW, SW, ED, SG, and RW thank the Simons Institute for their hospitality during the Summer 2019 program on the Foundations of Deep Learning. RW acknowledges funding from AFOSR and Facebook AI Research. This material is based upon work supported by the National Science Foundation under Grant No. DMS-1638352.

\bibliography{ref}

\begin{thebibliography}{66}
\providecommand{\natexlab}[1]{#1}
\providecommand{\url}[1]{\texttt{#1}}
\expandafter\ifx\csname urlstyle\endcsname\relax
  \providecommand{\doi}[1]{doi: #1}\else
  \providecommand{\doi}{doi: \begingroup \urlstyle{rm}\Url}\fi

\bibitem[Ali et~al.(2018)Ali, Kolter, and Tibshirani]{ali2018continuous}
Alnur Ali, J~Zico Kolter, and Ryan~J Tibshirani.
\newblock A continuous-time view of early stopping for least squares
  regression.
\newblock \emph{arXiv preprint arXiv:1810.10082}, 2018.

\bibitem[Arora et~al.(2020)Arora, Bartlett, Mianjy, and
  Srebro]{arora2020dropout}
Raman Arora, Peter Bartlett, Poorya Mianjy, and Nathan Srebro.
\newblock Dropout: Explicit forms and capacity control.
\newblock \emph{arXiv preprint arXiv:2003.03397}, 2020.

\bibitem[Arora et~al.(2018)Arora, Li, and Lyu]{arora2018theoretical}
Sanjeev Arora, Zhiyuan Li, and Kaifeng Lyu.
\newblock Theoretical analysis of auto rate-tuning by batch normalization.
\newblock \emph{arXiv preprint arXiv:1812.03981}, 2018.

\bibitem[Arora et~al.(2019)Arora, Cohen, Hu, and Luo]{arora2019implicit}
Sanjeev Arora, Nadav Cohen, Wei Hu, and Yuping Luo.
\newblock Implicit regularization in deep matrix factorization.
\newblock \emph{arXiv preprint arXiv:1905.13655}, 2019.

\bibitem[Ba et~al.(2016)Ba, Kiros, and Hinton]{ba2016layer}
Jimmy~Lei Ba, Jamie~Ryan Kiros, and Geoffrey~E Hinton.
\newblock Layer normalization.
\newblock \emph{arXiv preprint arXiv:1607.06450}, 2016.

\bibitem[Bartlett et~al.(2019)Bartlett, Long, Lugosi, and Tsigler]{BLLT19}
Peter~L Bartlett, Philip~M Long, G{\'a}bor Lugosi, and Alexander Tsigler.
\newblock Benign overfitting in linear regression.
\newblock \emph{arXiv preprint arXiv:1906.11300}, 2019.

\bibitem[Belkin et~al.(2019)Belkin, Hsu, and Xu]{belkin2019two}
Mikhail Belkin, Daniel Hsu, and Ji~Xu.
\newblock Two models of double descent for weak features.
\newblock \emph{arXiv preprint arXiv:1903.07571}, 2019.

\bibitem[Cai et~al.(2019)Cai, Li, and Shen]{cai2019quantitative}
Yongqiang Cai, Qianxiao Li, and Zuowei Shen.
\newblock A quantitative analysis of the effect of batch normalization on
  gradient descent.
\newblock In \emph{International Conference on Machine Learning}, pages
  882--890, 2019.

\bibitem[Cand{\`e}s and Recht(2009)]{candes2009exact}
Emmanuel~J Cand{\`e}s and Benjamin Recht.
\newblock Exact matrix completion via convex optimization.
\newblock \emph{Foundations of Computational mathematics}, 9\penalty0
  (6):\penalty0 717, 2009.

\bibitem[Chen et~al.(2019)Chen, Dobriban, and Lee]{chen2019invariance}
Shuxiao Chen, Edgar Dobriban, and Jane~H Lee.
\newblock Invariance reduces variance: Understanding data augmentation in deep
  learning and beyond.
\newblock \emph{arXiv preprint arXiv:1907.10905}, 2019.

\bibitem[Dobriban and Wager(2018)]{dobriban2018high}
Edgar Dobriban and Stefan Wager.
\newblock High-dimensional asymptotics of prediction: Ridge regression and
  classification.
\newblock \emph{The Annals of Statistics}, 46\penalty0 (1):\penalty0 247--279,
  2018.

\bibitem[Donoho et~al.(2013)Donoho, Gavish, and Montanari]{donoho2013phase}
David~L Donoho, Matan Gavish, and Andrea Montanari.
\newblock The phase transition of matrix recovery from gaussian measurements
  matches the minimax mse of matrix denoising.
\newblock \emph{Proceedings of the National Academy of Sciences}, 110\penalty0
  (21):\penalty0 8405--8410, 2013.

\bibitem[Douglas et~al.(2000)Douglas, Amari, and Kung]{douglas2000gradient}
Scott~C Douglas, Shun-ichi Amari, and S-Y Kung.
\newblock On gradient adaptation with unit-norm constraints.
\newblock \emph{IEEE Transactions on Signal Processing}, 48\penalty0
  (6):\penalty0 1843--1847, 2000.

\bibitem[Du et~al.(2018)Du, Zhai, Poczos, and Singh]{du2018gradient}
Simon~S Du, Xiyu Zhai, Barnabas Poczos, and Aarti Singh.
\newblock Gradient descent provably optimizes over-parameterized neural
  networks.
\newblock \emph{arXiv preprint arXiv:1810.02054}, 2018.

\bibitem[Dukler et~al.(2020)Dukler, Montufar, and Gu]{dukler2020optimization}
Yonatan Dukler, Guido Montufar, and Quanquan Gu.
\newblock Optimization theory for relu neural networks trained with
  normalization layers.
\newblock In \emph{Proceedings of the 37th International Conference on Machine
  Learning}, 2020.

\bibitem[Gal and Ghahramani(2016)]{gal2016dropout}
Yarin Gal and Zoubin Ghahramani.
\newblock Dropout as a bayesian approximation: Representing model uncertainty
  in deep learning.
\newblock In \emph{international conference on machine learning}, pages
  1050--1059, 2016.

\bibitem[Ge et~al.(2016)Ge, Lee, and Ma]{ge2016matrix}
Rong Ge, Jason~D Lee, and Tengyu Ma.
\newblock Matrix completion has no spurious local minimum.
\newblock In \emph{Advances in Neural Information Processing Systems}, pages
  2973--2981, 2016.

\bibitem[Ge et~al.(2017)Ge, Jin, and Zheng]{ge2017no}
Rong Ge, Chi Jin, and Yi~Zheng.
\newblock No spurious local minima in nonconvex low rank problems: A unified
  geometric analysis.
\newblock In \emph{International Conference on Machine Learning}, pages
  1233--1242, 2017.

\bibitem[Glorot and Bengio(2010)]{glorot2010understanding}
Xavier Glorot and Yoshua Bengio.
\newblock Understanding the difficulty of training deep feedforward neural
  networks.
\newblock In \emph{Proceedings of the thirteenth international conference on
  artificial intelligence and statistics}, pages 249--256, 2010.

\bibitem[Goyal et~al.(2017)Goyal, Doll{\'a}r, Girshick, Noordhuis, Wesolowski,
  Kyrola, Tulloch, Jia, and He]{goyal2017accurate}
Priya Goyal, Piotr Doll{\'a}r, Ross Girshick, Pieter Noordhuis, Lukasz
  Wesolowski, Aapo Kyrola, Andrew Tulloch, Yangqing Jia, and Kaiming He.
\newblock Accurate, large minibatch sgd: Training imagenet in 1 hour.
\newblock \emph{arXiv preprint arXiv:1706.02677}, 2017.

\bibitem[Gunasekar et~al.(2017)Gunasekar, Woodworth, Bhojanapalli, Neyshabur,
  and Srebro]{gunasekar2017implicit}
Suriya Gunasekar, Blake~E Woodworth, Srinadh Bhojanapalli, Behnam Neyshabur,
  and Nati Srebro.
\newblock Implicit regularization in matrix factorization.
\newblock In \emph{Advances in Neural Information Processing Systems}, pages
  6151--6159, 2017.

\bibitem[Gunasekar et~al.(2018)Gunasekar, Lee, Soudry, and
  Srebro]{gunasekar2018characterizing}
Suriya Gunasekar, Jason Lee, Daniel Soudry, and Nathan Srebro.
\newblock Characterizing implicit bias in terms of optimization geometry.
\newblock \emph{arXiv preprint arXiv:1802.08246}, 2018.

\bibitem[Hastie et~al.(2019)Hastie, Montanari, Rosset, and
  Tibshirani]{hastie2019surprises}
Trevor Hastie, Andrea Montanari, Saharon Rosset, and Ryan~J Tibshirani.
\newblock Surprises in high-dimensional ridgeless least squares interpolation.
\newblock \emph{arXiv preprint arXiv:1903.08560}, 2019.

\bibitem[Hayes(2009)]{hayes2009statistical}
Monson~H Hayes.
\newblock \emph{Statistical digital signal processing and modeling}.
\newblock John Wiley \& Sons, 2009.

\bibitem[Haykin(2005)]{haykin2005adaptive}
Simon~S Haykin.
\newblock \emph{Adaptive filter theory}.
\newblock Pearson Education India, 2005.

\bibitem[Haykin and Widrow(2002)]{haykin2002least}
Simon~Saher Haykin and Bernard Widrow.
\newblock \emph{Least-mean-square adaptive filters}.
\newblock Citeseer, 2002.

\bibitem[He et~al.(2015)He, Zhang, Ren, and Sun]{he2015delving}
Kaiming He, Xiangyu Zhang, Shaoqing Ren, and Jian Sun.
\newblock Delving deep into rectifiers: Surpassing human-level performance on
  imagenet classification.
\newblock In \emph{Proceedings of the IEEE international conference on computer
  vision}, pages 1026--1034, 2015.

\bibitem[Helmke and Moore(2012)]{helmke2012optimization}
Uwe Helmke and John~B Moore.
\newblock \emph{Optimization and dynamical systems}.
\newblock Springer Science \& Business Media, 2012.

\bibitem[Hoffer et~al.(2018)Hoffer, Banner, Golan, and Soudry]{hoffer2018norm}
Elad Hoffer, Ron Banner, Itay Golan, and Daniel Soudry.
\newblock Norm matters: efficient and accurate normalization schemes in deep
  networks.
\newblock In \emph{Advances in Neural Information Processing Systems}, pages
  2160--2170, 2018.

\bibitem[Ioffe and Szegedy(2015)]{ioffe2015batch}
Sergey Ioffe and Christian Szegedy.
\newblock Batch normalization: Accelerating deep network training by reducing
  internal covariate shift.
\newblock \emph{arXiv preprint arXiv:1502.03167}, 2015.

\bibitem[Jacot et~al.(2018)Jacot, Gabriel, and Hongler]{jacot2018neural}
Arthur Jacot, Franck Gabriel, and Cl{\'e}ment Hongler.
\newblock Neural tangent kernel: Convergence and generalization in neural
  networks.
\newblock In \emph{Advances in neural information processing systems}, pages
  8571--8580, 2018.

\bibitem[Ji and Telgarsky(2019)]{ji2019implicit}
Ziwei Ji and Matus Telgarsky.
\newblock The implicit bias of gradient descent on nonseparable data.
\newblock In \emph{Conference on Learning Theory}, pages 1772--1798, 2019.

\bibitem[Kohler et~al.(2019)Kohler, Daneshmand, Lucchi, Hofmann, Zhou, and
  Neymeyr]{pmlr-v89-kohler19a}
Jonas Kohler, Hadi Daneshmand, Aurelien Lucchi, Thomas Hofmann, Ming Zhou, and
  Klaus Neymeyr.
\newblock Exponential convergence rates for batch normalization: The power of
  length-direction decoupling in non-convex optimization.
\newblock In \emph{AISTATS}, pages 806--815, 2019.

\bibitem[LeCun et~al.(2015)LeCun, Bengio, and Hinton]{lecun2015deep}
Yann LeCun, Yoshua Bengio, and Geoffrey Hinton.
\newblock Deep learning.
\newblock \emph{nature}, 521\penalty0 (7553):\penalty0 436--444, 2015.

\bibitem[Lee et~al.(2019)Lee, Xiao, Schoenholz, Bahri, Novak, Sohl-Dickstein,
  and Pennington]{lee2019wide}
Jaehoon Lee, Lechao Xiao, Samuel Schoenholz, Yasaman Bahri, Roman Novak, Jascha
  Sohl-Dickstein, and Jeffrey Pennington.
\newblock Wide neural networks of any depth evolve as linear models under
  gradient descent.
\newblock In \emph{Advances in neural information processing systems}, pages
  8570--8581, 2019.

\bibitem[Lee et~al.(2016)Lee, Simchowitz, Jordan, and Recht]{lee2016gradient}
Jason~D Lee, Max Simchowitz, Michael~I Jordan, and Benjamin Recht.
\newblock Gradient descent converges to minimizers.
\newblock \emph{arXiv preprint arXiv:1602.04915}, 2016.

\bibitem[Li et~al.(2018)Li, Ma, and Zhang]{li2018algorithmic}
Yuanzhi Li, Tengyu Ma, and Hongyang Zhang.
\newblock Algorithmic regularization in over-parameterized matrix sensing and
  neural networks with quadratic activations.
\newblock In \emph{Conference On Learning Theory}, pages 2--47. PMLR, 2018.

\bibitem[Lian and Liu(2019)]{pmlr-v89-lian19a}
Xiangru Lian and Ji~Liu.
\newblock Revisit batch normalization: New understanding and refinement via
  composition optimization.
\newblock In Kamalika Chaudhuri and Masashi Sugiyama, editors,
  \emph{Proceedings of Machine Learning Research}, volume~89 of
  \emph{Proceedings of Machine Learning Research}, pages 3254--3263, 16--18 Apr
  2019.

\bibitem[Liang and Rakhlin(2018)]{liang2018just}
Tengyuan Liang and Alexander Rakhlin.
\newblock Just interpolate: Kernel "ridgeless" regression can generalize.
\newblock \emph{arXiv preprint arXiv:1808.00387}, 2018.

\bibitem[Liu and Dobriban(2019)]{liu2019ridge}
Sifan Liu and Edgar Dobriban.
\newblock Ridge regression: Structure, cross-validation, and sketching.
\newblock \emph{arXiv preprint arXiv:1910.02373}, 2019.

\bibitem[Luo et~al.(2019)Luo, Wang, Shao, and Peng]{luo2018towards}
Ping Luo, Xinjiang Wang, Wenqi Shao, and Zhanglin Peng.
\newblock Towards understanding regularization in batch normalization.
\newblock In \emph{International Conference on Learning Representations}, 2019.

\bibitem[Mahoney and Martin(2019)]{mahoney2019traditional}
Michael Mahoney and Charles Martin.
\newblock Traditional and heavy tailed self regularization in neural network
  models.
\newblock In \emph{International Conference on Machine Learning}, pages
  4284--4293, 2019.

\bibitem[Mahoney(2012)]{mahoney2012approximate}
Michael~W Mahoney.
\newblock Approximate computation and implicit regularization for very
  large-scale data analysis.
\newblock In \emph{Proceedings of the 31st ACM SIGMOD-SIGACT-SIGAI symposium on
  Principles of Database Systems}, pages 143--154. ACM, 2012.

\bibitem[Martin and Mahoney(2018)]{martin2018implicit}
Charles~H Martin and Michael~W Mahoney.
\newblock Implicit self-regularization in deep neural networks: Evidence from
  random matrix theory and implications for learning.
\newblock \emph{arXiv preprint arXiv:1810.01075}, 2018.

\bibitem[Mianjy et~al.(2018)Mianjy, Arora, and Vidal]{mianjy2018implicit}
Poorya Mianjy, Raman Arora, and Rene Vidal.
\newblock On the implicit bias of dropout.
\newblock \emph{arXiv preprint arXiv:1806.09777}, 2018.

\bibitem[Morgan and Bourlard(1990)]{morgan1990generalization}
Nelson Morgan and Herv{\'e} Bourlard.
\newblock Generalization and parameter estimation in feedforward nets: Some
  experiments.
\newblock In \emph{Advances in neural information processing systems}, pages
  630--637, 1990.

\bibitem[Neyshabur et~al.(2014)Neyshabur, Tomioka, and
  Srebro]{neyshabur2014search}
Behnam Neyshabur, Ryota Tomioka, and Nathan Srebro.
\newblock In search of the real inductive bias: On the role of implicit
  regularization in deep learning.
\newblock \emph{arXiv preprint arXiv:1412.6614}, 2014.

\bibitem[Neyshabur et~al.(2019)Neyshabur, Li, Bhojanapalli, LeCun, and
  Srebro]{neyshabur2018the}
Behnam Neyshabur, Zhiyuan Li, Srinadh Bhojanapalli, Yann LeCun, and Nathan
  Srebro.
\newblock The role of over-parametrization in generalization of neural
  networks.
\newblock In \emph{International Conference on Learning Representations}, 2019.
\newblock URL \url{https://openreview.net/forum?id=BygfghAcYX}.

\bibitem[Poggio et~al.(2019)Poggio, Banburski, and Liao]{PBL19}
Tomaso Poggio, Andrzej Banburski, and Qianli Liao.
\newblock Theoretical issues in deep networks: Approximation, optimization and
  generalization.
\newblock \emph{arXiv preprint arXiv:1908.09375}, 2019.

\bibitem[Proakis(2001)]{proakis2001digital}
John~G Proakis.
\newblock \emph{Digital signal processing: principles algorithms and
  applications}.
\newblock Pearson Education India, 2001.

\bibitem[Rockafellar and Wets(2009)]{rockafellar2009variational}
R~Tyrrell Rockafellar and Roger J-B Wets.
\newblock \emph{Variational analysis}, volume 317.
\newblock Springer Science \& Business Media, 2009.

\bibitem[Salimans and Kingma(2016)]{salimans2016weight}
Tim Salimans and Diederik~P Kingma.
\newblock Weight normalization: A simple reparameterization to accelerate
  training of deep neural networks.
\newblock In \emph{Advances in Neural Information Processing Systems}, pages
  901--909, 2016.

\bibitem[Santurkar et~al.(2018)Santurkar, Tsipras, Ilyas, and
  Madry]{santurkar2018does}
Shibani Santurkar, Dimitris Tsipras, Andrew Ilyas, and Aleksander Madry.
\newblock How does batch normalization help optimization?
\newblock In \emph{Advances in Neural Information Processing Systems}, pages
  2483--2493, 2018.

\bibitem[Soudry et~al.(2018)Soudry, Hoffer, Nacson, Gunasekar, and
  Srebro]{soudry2018implicit}
Daniel Soudry, Elad Hoffer, Mor~Shpigel Nacson, Suriya Gunasekar, and Nathan
  Srebro.
\newblock The implicit bias of gradient descent on separable data.
\newblock \emph{The Journal of Machine Learning Research}, 19\penalty0
  (1):\penalty0 2822--2878, 2018.

\bibitem[Strand(1974)]{strand1974theory}
Otto~Neall Strand.
\newblock Theory and methods related to the singular-function expansion and
  landweber’s iteration for integral equations of the first kind.
\newblock \emph{SIAM Journal on Numerical Analysis}, 11\penalty0 (4):\penalty0
  798--825, 1974.

\bibitem[Strohmer and Vershynin(2009)]{strohmer2009randomized}
Thomas Strohmer and Roman Vershynin.
\newblock A randomized kaczmarz algorithm with exponential convergence.
\newblock \emph{Journal of Fourier Analysis and Applications}, 15\penalty0
  (2):\penalty0 262, 2009.

\bibitem[Tian(2019)]{tian2019over}
Yuandong Tian.
\newblock Over-parameterization as a catalyst for better generalization of deep
  relu network.
\newblock \emph{arXiv preprint arXiv:1909.13458}, 2019.

\bibitem[Tian et~al.(2019)Tian, Jiang, Gong, and Morcos]{tian2019luck}
Yuandong Tian, Tina Jiang, Qucheng Gong, and Ari Morcos.
\newblock Luck matters: Understanding training dynamics of deep relu networks.
\newblock \emph{arXiv preprint arXiv:1905.13405}, 2019.

\bibitem[Ulyanov et~al.(2016)Ulyanov, Vedaldi, and
  Lempitsky]{ulyanov2016instance}
Dmitry Ulyanov, Andrea Vedaldi, and Victor Lempitsky.
\newblock Instance normalization: The missing ingredient for fast stylization.
\newblock \emph{arXiv preprint arXiv:1607.08022}, 2016.

\bibitem[Va{\v{s}}kevi{\v{c}}ius et~al.(2019)Va{\v{s}}kevi{\v{c}}ius, Kanade,
  and Rebeschini]{vavskevivcius2019implicit}
Tomas Va{\v{s}}kevi{\v{c}}ius, Varun Kanade, and Patrick Rebeschini.
\newblock Implicit regularization for optimal sparse recovery.
\newblock \emph{arXiv preprint arXiv:1909.05122}, 2019.

\bibitem[Vershynin(2018)]{vershynin2018high}
Roman Vershynin.
\newblock \emph{High-dimensional probability: An introduction with applications
  in data science}, volume~47.
\newblock Cambridge University Press, 2018.

\bibitem[Wager et~al.(2013)Wager, Wang, and Liang]{wager2013dropout}
Stefan Wager, Sida Wang, and Percy~S Liang.
\newblock Dropout training as adaptive regularization.
\newblock In \emph{Advances in neural information processing systems}, pages
  351--359, 2013.

\bibitem[Ward et~al.(2019)Ward, Wu, and Bottou]{pmlr-v97-ward19a}
Rachel Ward, Xiaoxia Wu, and Leon Bottou.
\newblock {A}da{G}rad stepsizes: Sharp convergence over nonconvex landscapes.
\newblock In \emph{Proceedings of the 36th International Conference on Machine
  Learning}, pages 6677--6686, 09--15 Jun 2019.

\bibitem[Wu et~al.(2018)Wu, Ward, and Bottou]{wu2018wngrad}
Xiaoxia Wu, Rachel Ward, and L{\'e}on Bottou.
\newblock Wngrad: learn the learning rate in gradient descent.
\newblock \emph{arXiv preprint arXiv:1803.02865}, 2018.

\bibitem[Xu et~al.(2019)Xu, Sun, Zhang, Zhao, and Lin]{xu2019understanding}
Jingjing Xu, Xu~Sun, Zhiyuan Zhang, Guangxiang Zhao, and Junyang Lin.
\newblock Understanding and improving layer normalization.
\newblock In \emph{Advances in Neural Information Processing Systems}, pages
  4383--4393, 2019.

\bibitem[Zhang et~al.(2016)Zhang, Bengio, Hardt, Recht, and
  Vinyals]{zhang2016understanding}
Chiyuan Zhang, Samy Bengio, Moritz Hardt, Benjamin Recht, and Oriol Vinyals.
\newblock Understanding deep learning requires rethinking generalization.
\newblock \emph{arXiv preprint arXiv:1611.03530}, 2016.

\end{thebibliography}
\bibliographystyle{plainnat} 
\newpage
\appendix
\newpage

\section{Adaptive Regularization} \label{sec:adareg}


We illustrate that the two Algorithms  can heuristically be viewed as GD on adaptively $\ell_2$-regularized regression problems. 
The regularization parameter changes for each iteration in the algorithms:

\begin{equation*}
   \min_{x_{t+1}} \left(\frac{1}{2}\|Ax_{t+1}-Ax^*\|^2 + \lambda(w_t,g_t,\eta,A,y)\|x_{t+1}\|^2\right).
\end{equation*}


However, it is difficult to characterize the behavior of $\lambda_t$ in general. 

{\bf WN.} Let $x_t=\frac{g_tw_t}{\|w_t\|}$. Notice that:
\begin{align*}
    \frac{g_{t+1} w_{t+1}}{\|w_{t+1}\|}
    = & \frac{g_t}{\|w_t\|} w_t + (\frac{g_{t+1}}{\|w_{t+1}\|} - \frac{g_t}{\|w_t\|})w_t  - \eta_t \frac{g_{t+1}}{\|w_{t+1}\|} \frac{g_t}{\|w_t\|} (I - \frac{w_t w_t^T}{\|w_t\|^2}) A^T r
\end{align*}

This can be translated to the update of $x_t$ as
\begin{align*}
    x_{t+1} = & x_t - \eta_t \frac{g_{t+1}g_t}{\|w_{t+1}\|\|w_t\|} A^T r - \left( 1 -\frac{g_{t+1}\|w_t\|}{g_t\|w_{t+1}\|} 
     - \eta_t \frac{g_{t+1}}{\|w_{t+1}\|} \langle \frac{w_t}{\|w_t\|^2}, A^T r\rangle \right) x_t
\end{align*}

{\bf rPGD.} Let $x_t=g_tw_t$. The update of $w_t$ in Algorithm~\ref{alg:main} is
\begin{equation}
    w_{t+1} = \frac{1}{\|v_t\|}(w_t -\eta_tg_tA^T(Ax_t -Ax^*)).
\end{equation}
We can now write the update of $x_{t+1} = g_{t+1}w_{t+1}$ as
$$
x_{t+1} =  x_t - \frac{\eta_tg_tg_{t+1}}{\|v_t\|} A^T(Ax_t-Ax^*) - (1-\frac{g_{t+1}}{g_t\|v_t\|})x_t. 
$$
Both updates can be viewed as a gradient step on the following $\ell_2$-regularized regression problem, with specific choices of $\lambda_t$ at iteration $t$:
\begin{equation*}
    \frac{1}{2}\|Ax_{t+1}-Ax^*\|^2 + \lambda_t\|x_{t+1}\|^2,
\end{equation*}
We see that the regularization parameter changes for each iteration for both WN and rPGD, as follows: 
\begin{align*}
   \text{(WN)} \quad \lambda_t =& \frac{\|w_{t+1}\|\|w_t\|}{g_{t+1} g_t} (1 -\frac{g_{t+1}\|w_t\|}{g_t\|w_{t+1}\|} 
     - \eta \frac{g_{t+1}}{\|w_{t+1}\|} \langle \frac{w_t}{\|w_t\|^2}, A^T r\rangle)\\
   \text{(rPGD)} \quad \lambda_t =&   (g_t\|v_t\|-g_{t+1})/(\eta_tg_t^2g_{t+1}).
\end{align*}
 The regularization parameter $\lambda_t$ is highly dependent on $g_t$, $g_{t+1}$ and the input matrix $A$. However, it is difficult to characterize the behavior of $\lambda_t$ in general. In particular, we require the parameters $g_t$, $g_{t+1}$, $w_t$ and $w_{t+1}$ updated in a way that $\lambda_t>0$. 
For the simpler setting of orthogonal $A$, we can see for rPGD that: 1) If the learning rate of $g$ is small enough, we will have $g_{t+1}<g_t\|v_t\|$, which means that $\lambda_t>0$; 2) When $g_tw_t$ is close to $g^*w^*$, we will have $\|v_t\|\approx 1$, and $g_{t+1}\approx g_t$, which means that $\lambda_t\approx 0$.

\section{Proof of Lemma~\ref{cor:flow}}
\label{pf:flow}

\begin{proof}
{\bf rPGD.} First we start with the reparametrized Projected Gradient Descent algorithm. The gradients for rPGD are
\begin{equation*}
    \nabla_{w}f(w, g) = g \nabla L(gw),\;  \nabla_{g}f(w, g) = w^T\nabla L(gw),
\end{equation*}
First, the gradient step on $g$ clearly leads to the gradient flow for $g$. Second, for the update on $w$, we expand all terms to first order in $\eta$. Let $a_t=\nabla_{w_t}f(w_t, g_t)$. We start by expanding the squared Euclidean norm
\begin{align*}
\|w_t - \eta\nabla_{w_t}f(w_t, g_t)\|_2^2 &= \|w_t - \eta a_t\|_2^2\\
&=  \|w_t\|^2 - 2\eta w_t^\top a_t + \eta^2 \|a_t\|_2^2\\
&=  1 - 2\eta w_t^\top a_t + O(\eta^2).
\end{align*}

On the last line, we have used that the iterates are normalized, so $\|w_t\|^2=1$.

Now, we can use the expansion $(1+x)^{1/2} = 1+ x/2 + O(x^2)$, valid for $|x|\ll1$,  on the right hand side of the above display, to get
\begin{align*}
\|w_t - \eta\nabla_{w_t}f(w_t, g_t)\|_2
&=  1 - \eta w_t^\top a_t + O(\eta^2).
\end{align*}

Next we use this expansion in the update rule for the weights:
\begin{align*}
&w_{t+1} = \frac{w_t - \eta \nabla_{w_t}f(w_t, g_t)}
{\|w_t - \eta\nabla_{w_t}f(w_t, g_t)\|_2}\\
&= \frac{w_t - \eta a_t}
{1 - \eta w_t^\top a_t + O(\eta^2)}
\\
&= (w_t - \eta a_t) \cdot (1 + \eta w_t^\top a_t+ O(\eta^2)).
\end{align*}
In the last line, we have used the expansion
\begin{align*}
\frac{1}{1 - \eta x + O(\eta^2)}
&= 1 + \eta x+ O(\eta^2)
\end{align*}
valid for $|\eta|\ll1$ and $x$ of a constant order. Recall now that for an arbitrary vector $w$, we defined $P_{w^\perp} = I-\frac{ww^\top}{\|w\|_2^2}$ as the projection into the orthocomplement of $w$. Since $\|w_t\|=1$, we have $\Pw = I-w_tw_t^\top$. By expanding the product and rearranging, keeping only the terms of larger order than $\eta$, we find
\begin{align*}
&w_{t+1}
= w_t - \eta a_t+ \eta w_t w_t^\top a_t+ O(\eta^2)
\\
&= w_t - \eta \Pw a_t+ O(\eta^2).
\end{align*}

Taking $\eta\to 0$ and substituting the expression for $a_t$ and $\nabla_{w}f(w, g)$, 
we obtain the rPGD flow dynamics for $w$, i.e., $\dot w_t= -g_t \Pw\nabla L (g_tw_t)$. The update rule for $g$ follows directly.

{\bf WN.} We now study weight normalization \cite{salimans2016weight}. The WN objective function can be written using the loss function $L(x) = \|Ax-y\|^2/2$ as
$$h(g,w) =
L\left(g \frac{w}{\|w\|}\right).
$$
The discrete time algorithm is thus updated as
\begin{align*}
 v_t &= w_t/\|w_t\|\\
 r_t &= y-g_t A v_t\\
 g_{t+1} &= g_t- c\eta \cdot\langle{v_t, \nabla L\left(g_t\frac{w_t}{\|w_t\|} \right) \rangle}\\
 w_{t+1} &= w_t - \eta \cdot g_t \cdot \Pw \frac{1}{\|w_t\|}\nabla L(g_t \frac{w_t}{\|w_t\|}).
\end{align*}
When $\eta\to 0$, we recover the gradient flow on $g_t$ and $w_t$, i.e., recalling $v_t= w_t/\|w_t\|$
\begin{align*}
 \dot g_{t} &=-c\cdot\langle{v_t,\nabla L(g_t v_t)\rangle}\\
 \dot w_{t} &=- g_t \cdot \Pw \frac{1}{\|w_t\|}\nabla L(g_t v_t)=- \frac{g_t}{\|w_t\|} \cdot \Pw \nabla L(g_t v_t).
\end{align*}
Note the fact $\frac{d\|w_t\|^2}{dt}=2w_t^T\dot w_t=0$
which gives $\|w_t\|=\|w_0\| = 1$.


Hence, for WN with initialization $\|w_0\|=1$, we have that $g_t,v_t$ evolves exactly equivalent to the rPGD flow. The final formula for WN and rPGD that we will analyze is:
\begin{align*}
 \dot g_{t} &=-c\cdot w_t^T A^T r_t\\
 \dot w_{t} &=- g_t \cdot \Pw A^Tr_t.
\end{align*}
\end{proof}

\section{Remaining Proofs for Section \ref{sec:main}}
\label{pf:rpgdf}

\subsection{Proof of Lemma \ref{lem:stat_p_loss}}
\begin{proof}
At a stationary points of the loss, we have, with $r = y - Agw$
\begin{align*}
 \partial_g h(w,g) &=w^T A^T r = 0\\
 \partial_w h(w,g) &= g \cdot P_{w^\perp} A^Tr = 0.
\end{align*}
If $g\neq 0$, then we get $P_{w^\perp} A^Tr = 0$. By adding this up with the first equation, we get $A^Tr = 0$. Using that the smallest eigenvalue of $AA^T$ is nonzero, we conclude that $r=0$. Hence this is stationary point with zero loss, which is also a global minimum.

Else, if $g=0$, we see that $r = y - Agw = y$. Hence in this case, the stationary points belong to the set $S:= \{(g,w): g =0, y^T Aw = 0\}.$ This finishes the proof.
\end{proof}

\subsection{Proof of Lemma \ref{lem:loss} (dynamics of loss $\|r_t\|$)} 
We have
\begin{align*}
d[1/2 \|r_t\|^2]/dt &= r_t^T \dot r_t 
= r_t^T A d(g_t w_t)/dt\\
&= r_t^T A [\dot g_t w_t+ g_t \dot w_t]\\
&= - r_t^T A[c\cdot w_tw_t^T A^T r_t+ g_t^2 \Pw A^T r_t]\\
&= - r_t^T A[c\cdot w_tw_t^T+ g_t^2\Pw] A^T r_t
\end{align*}
Thus, 
\begin{align*}
d[1/2\|r_t\|^2]/dt \leq -\min\{g_t^2, c\}\|A^Tr_t\|^2,\quad d[1/2\|r_t\|^2]/dt\geq -\max\{g_t^2, c\}\|A^Tr_t\|^2.
\end{align*}
We get a geometric convergence of the loss to zero, as soon as we can get a lower bound on $g_t^2$, which will be discussed below.  If we have $g_t^2\ge C^2$ for some constant $C^2>0$, we have
\begin{align*}
\min(g_t^2,c) r_t^T A A^T r_t
\ge \min(C^2,c) \lambda_{\min}(A A^T) \|r_t\|^2
\end{align*}
and so with $k:=\min(C^2,c) \lambda_{\min}(A A^T)$,
\begin{align*}
d[1/2 \cdot \|r_t\|^2]/dt &\le - k \|r_t\|^2 \quad \Rightarrow \quad \|r_t\|^2\le
 \exp(-kt) \|r_0\|^2.
\end{align*}

\subsection{Proof of Lemma \ref{lemma:wperp}}


Define $P^\perp $ the projection into the orthocomplement of the row span of $A$,  hence $w^\perp = P^\perp w$ and $P^\perp A^T=0$. For simplicity, write $h(w_t,g_t) =h_t$ 
        \begin{equation*}
        \begin{split}
         \frac{d{w_t^\perp}}{dt}    &= P^\perp \frac{d{w_t}}{dt} 
            = -\frac{g_t}{\norm{w_t}}P^\perp(I-\frac{w_tw_t^\top}{\norm{w_t}^2}) \nabla_w h_t
            = \frac{g_t}{\norm{w_t}} P^\perp\frac{w_tw_t^\top}{\norm{w_t}^2} \nabla_w h_t 
            =  \frac{g_t}{\norm{w_t}^2} P^\perp{w_t} \left(\nabla_{g} h_t\right)  \\
           & =  -\frac{g_t}{\norm{w_t}^2}  P^\perp {w_t}\left(\frac{1}{c} \frac{dg_t}{dt}\right) 
            =  -\frac{1}{2}\frac{P^\perp{w_t}}{\norm{w_t}^2 c}\frac{d g^2_t}{dt} = -\frac{1}{2}\frac{w_t^\perp}{\norm{w_0}^2c}\frac{d g^2_t}{dt}.
        \end{split}
        \end{equation*}
The last equality due to the fact that
    ${d\|w_t\|^2}/{dt}=2w_t^T\dot w_t=0$, which is also observed in \citep{tian2019luck,tian2019over}.
Solving the dynamics $\frac{d{w_t^\perp}}{dt} =-\frac{1}{2}\frac{w_t^\perp}{c\norm{w_0}^2}\frac{d g^2_t}{dt} $ with $\|w_0\|=1$ results in \eqref{eq:inv1}.

\subsection{Proof of Theorem~\ref{thm:convergence} (Convergence in the general case)}

\subsubsection{Proof that either the loss converges to zero, or the iterates $(g_t,w_t)$ converge to the stationary set $\mathcal{S}$ defined in Lemma \ref{lem:stat_p_loss}.}


We start with the ODE for the loss, 
\begin{align}\label{ol}
d[1/2 \|r_t\|^2]/dt
&= - r_t^T A[c\cdot w_tw_t^T+ g_t^2\Pw] A^T r_t.
\end{align}
This shows that the loss is decreasing, possibly not strictly.

If we have $g_t^2$ bounded away from zero, then the loss converges to zero geometrically. Thus, the only remaining case is when $g_t \to 0$.

Now, we have that the iterates take the form $x_t = g_t w_t$, and $\|w_t\|=1$ are bounded. Hence, as $g_t\to 0$, we must have $x_t \to 0$.

Since the loss is continuous, we also have $\|r_t\|^2 = \|y- Ag_t w_t\|^2 \to \|y\|^2$.

Suppose that the loss does not converge to zero. Since the loss is decreasing, this means that $\|r_t\|\to c>0$ for some constant $c$. 

From Equation \eqref{ol}, this implies that 
\begin{align*}
(r_t^T Aw_t)^2 &\to 0.
\end{align*}
Else, if this quantity is bounded away from zero, then $d[1/2 \|r_t\|^2]/dt<-c'$ for some $c'>0$, which would show $\|r_t\|$ decreases unboundedly, and is a contradiction.



Thus, we conclude that if the residual $r_t$ does not converge to zero, then the iterates $x_t = g_t w_t$ converge to zero: $x_t \to 0$. Moreover, $g_t \to 0$ and $y^T Aw_t \to 0$. Given that $w_t$ is bounded, this shows that $(g_t,w_t)$ converges to the \emph{set} of those stationary points of the loss characterized by 
$$S:= \{(g,w): g =0, y^T Aw = 0\}.$$
Note specifically that we have not shown that $w_t$ converges to a specific stationary point, but rather only that it converges to the set $S$ of stationary points given above.



This result does not give a rate of convergence, so it is weaker than the result when $g_t$ is bounded away from zero. However, it is also more general. Initializing such that the loss is less than the loss at zero can be achieved conveniently, because we can calculate the value of the loss.

\subsubsection{Part I: Characterizing the solution when $c>0$}

The characterization follows by tracking the dynamics of the components in the row span of $A$ and its orthocomplement separately. The component in the row span converges to $x^*$. The normalized component $w_t^\perp$ in the orthocomplement is characterized by the invariant from the prior lemma. The scale of that component converges to $g^*$, as $g_t \to g^*$. This gives the desired result.

\subsubsection{Part II: Fixed  $g_t$, i.e. $c=0$}
For the fixed $g_t$ case, we have
\begin{align*}
d[1/2 \cdot \|r_t\|^2]/dt &= r_t^T \dot r_t 
= r_t^T A d(g_t w_t)/dt\\
&= r_t^T A [\dot g_t w_t+ g_t \dot w_t]\\
&= r_t^T A g_0 \dot w_t\\
&= - r_t^T A g_0 \Pw g_0 A^T r_t\\
&= - g_0^2 \|\Pw A^T r_t\|^2.
\end{align*}
Now, it follows that $\|r_t\|$ is a non-increasing quantity, which is also strictly decreasing as long as $\|\Pw A^T r_t\|>0$. It also follows that as $t\to\infty$, we have $\|\Pw A^T r_t\|\to0$. Now, since $g_0<g^*$, it follows that $A^T r_t$ has a norm that is strictly bounded away from zero, i.e., $\|A^T r_t\|>c_0>0$ for some $c_0>0$. So, we do not have the residual going to zero in this case.
Hence, this can also be written as $w_t = c_t \cdot A^T r_t/\|A^T r_t\|$, for some sequence of scalars $c_t$ with $\liminf c_t^2>0$. 

Hence, $w_t$ becomes asymptotically parallel to the row space of $A$. Now, since $w_t$ lives on the compact space of unit vectors, considering any subsequence of it, it also follows that it has a convergent subsequence. Let $w$ be the limit along any convergent subsequence. It follows that $w=\pm A^T r/\|A^T r\|$. Next we note that the solution with $+$ actually \emph{maximizes} the loss over $\|w\|=1$. Hence, the only possible solution is $w=- A^T r/\|A^T r\|$. Since this holds for any convergent subsequence, it follows that $w_t$ itself must converge. 

Now we get a more explicit form for the solution $w$. We can say that $w$ is the unique unit norm vector such that $w=- \hat{c} A^T r$, for some $\hat{c}>0$. Then we can write that equation as
\begin{align*}
w &= - \hat{c}A^T(Agw-y)\\
(I + cgA^TA) w &= \hat{c}A^Ty\\
w &= (\hat{c}^{-1}I + gA^TA)^{-1} A^Ty.
\end{align*}
Thus, $w$ is the unique vector of the above form such that $\|w\|=1$. This can be viewed as a form of implicit regularization. Namely, $w$ is the unique vector, for which there is some regularization parameter $\hat{c}$ such that, $w$ minimizes the regularized least squares objective 
$$\frac12 \|Agw-y\|^2 +\frac{g}{2\hat{c}}\|w\|^2$$
and $\|w\|=1$. This $w$ will in general not be the pointing in the direction of the optimal solution. We recall that the optimal solution $w^*$ has the form $w^* = A^\dagger b/\|A^\dagger y\|$, where $A^\dagger$ is the pseudoinverse of $A$. We recall that the action of the pseudoinverse can be characterized as the limit of ridge regularization with infinitely small penalization, i.e., $A^\dagger y = \lim_{\lambda\to0}(A^TA+\lambda I)^{-1} A^Ty$. For orthogonal $A$, we see that $w$ converges to the right direction, because $A^T=A^\dagger$. However, for general $A$, the flow does not in general converge to the min-norm direction.

Now, suppose we start the flow for both $g_t,w_t$ again from a point $w_0$ that belongs to the span of the row space of $A$ (call it $R$). Then, by the update rule for $w_t$, it follows that $w_t\in R$ for all $t$. Therefore, the derivative of the loss becomes
\begin{align*}
d[1/2 \cdot \|r_t\|^2]/dt &= r_t^T \dot r_t 
= r_t^T A d(g_t w_t)/dt\\
&= r_t^T A [\dot g_t w_t+ g_t \dot w_t]\\
&= - r_t^T A[w_tw_t^T A^T r_t+ g_t^2 \Pw A^T r_t]\\
&= - (r_t^T Aw_t)^2.
\end{align*}
From arguments similar to before, it follows that $r_t^T Aw_t\to 0$ as $t\to\infty$. Moreover, from a similar subsequence argument it also follows that $w_t\to w$ such that $r^TAw=0$. Since $\|w\|=1$ and $A$ has full row rank, it follows that $r=0$. Hence the flow converges to a zero of the loss. Moreover, since $w\in R$, it follows that this is the minimum norm solution.

\subsection{Proof of Theorem~\ref{convergence_rate} (Convergence rate)}

Note first that by \eqref{eq:inv1}, we have 
\begin{align}
     g_t^2=
2\log\|w_0^\perp\|+g_0^2-2\log\|w_t^\perp\|\ge 2\log\|w_0^\perp\|+g_0^2 \label{eq:lowwer}
\end{align} 
where the last inequality is due to the fact $\|w_t^\perp\| \le \|w_t\|=1$.
Thus, we have our  lower bound. From Lemma \ref{lem:loss}, we get the convergence rate for  this case.

\section{Beyond Linear Regression}\label{sec:beyond}
Here we illustrate that the invariant in the optimization path holds more generally than for linear regression, and specifically for certain general loss functions that only depend on a small dimensional subspace of the parameter space.
Let $L:\R^d \to \R$ be the loss function, and our goal is to solve
\begin{align}
\min_{x\in \R^d} L(x)\label{eq:main}
\end{align}
where $L(x)$ is differentiable and satisfies Assumption \ref{asmp:ld}.
\begin{assumption}[Low-dimensional gradient]\label{asmp:ld}
There exists a projection matrix $P\in\R^{d\times d} $ with rank $r<d$  such that 
$$
     (I-P)\nabla L(x) = 0 , \forall x\in \R^d.
$$
\end{assumption}
Let $ P^\perp= (I-P)$.  Assumption \ref{asmp:ld} is equivalent to the fact that the gradient of $L$ lives in the low-dimensional space given by the span of $P$, 
$\nabla L(x) \in span(P)$. This implies 
$$L(x)=L(Px)  \quad  \forall x\in\R^d.$$  This means that the objective only depends on the projection of $x$ into the span of $P$. To use the orthogonal projection in what follows, define $x^{\parallel} = P x$ and $x^{\perp} = P^\perp x$. For the undetermined linear regression, $P=A^\dagger A$ where $A^\dagger$ is the pseudo-inverse of the matrix $A$.
\begin{theorem}[WN flow Invariance for General Loss]
\label{thm:implicit}\label{thm:general}
Consider the loss function in \eqref{eq:main} with Assumption  \ref{asmp:ld}. The WN method transforms the loss function to $ h(g,w) =  L \left(g_t\frac{w_t}{\|w_t\|}\right)$. The WN gradient flow from Algorithm \ref{alg:wn} with initial condition $(w_0,g_0)$, started from $w_0$ with not necessarily unit norm, has the invariant
 \[w_t^\perp= \exp\left(\frac{g_0^2-g_t^2}{2\|w_0\|^2}\right) w_0^\perp.\]
\end{theorem}
The proof of the above theorem is a simple extension of the proof in Part I of Lemma \ref{lemma:wperp} 
with $A^\top r_t=\nabla L \left(g_t\frac{w_t}{\|w_t\|}\right)$. This result suggests that the reason for the invariance is that the original objective function before reparametrization only depends on a smaller dimensional space.

\section{Remaining Proofs for Section \ref{sec:orthogonal}}
\subsection{Proof of Theorem \ref{thm:orthogonal1} Case (a)}
\label{sec:proof_orthogonal}
We restate the case (a) of Theorem \ref{thm:orthogonal1} here.
\begin{theorem}[Updating $g$ in Phase I] \label{partI}
 Suppose we  initialize with $g_0< g^*$. Let $\delta_0 = (g^*)^2 - (g_0)^2>0$, $0<\rho<1$ and $\delta< \frac{\varepsilon}{2g^*+\varepsilon}$. 
Suppose the number of iterations $T_1$ and $T_2$ is of the order:
\[
T_1= \left(1+\frac{(g^*)^2}{\rho\delta_0} \right)\log\left(\frac{1-\|Aw_0\|^2}{\delta^2}\right) = \mathcal{O}\left(\frac{(g^*)^2}{\rho\delta_0} \log\left( \frac{1}{\delta^2}\right)\right), 
\]
\[T_2 =\frac{1}{\gamma^{(2)}}\log \left(\frac{\rho \delta_0}{g^{*2}(1-\ep)\sqrt{\ep}}\right)
 = \mathcal{O}\left( \frac{1}{\gamma^{(2)}}\log\left(\frac{\rho (g^{*2}-g_0^2)}{g^{*2}\sqrt{\varepsilon}}\right) \right).\]
Fix $g_1=g_0$ at the first step. For iterations $t= 1, \ldots, T_1-1$, set the stepsize for $g$ to 
\[\gamma^{(1 )}< \min\left\{\frac{\|Ag_0w_0\|^2}{2 T_1((g^*)^2-\|Ag_0w_0\|^2)} \log\left((1-\rho)\frac{g^{*2}}{g_0^2}+\rho \right), \frac{1}{2(1+\|Aw_0\|^2)}\right\}\]
and to any $\gamma^{(2)}\leq \frac{1}{4}$ for $t= T_1,T_1+1,\ldots,T_1+T_2-1$. Set $\eta_t=1/g_t^2$. Then we reach $\|w_{T}^\perp\|\leq \ep$ and $\|Ag_{T}w_{T}-b \|^2\leq \varepsilon g^{*2}$ after $T=T_1+T_2$ iterations.
\end{theorem}


\begin{proof}
\textbf{The norm $\|w^\perp\|$.} Let us first get the upper bound of $g_{t+1}$ to see how the norm of $\|w_t^\perp\|= 1-\|Aw_t\|^2$ evolves. Suppose that we have for some $0<\rho<1$ \begin{align}
    (g^*)^2 - \|Ag_{T_1} w_{T_1}\|^2\geq (g^*)^2-g_{T_1}^2 = \rho \delta_0\label{eq:delta}
\end{align}
By Lemma \ref{lem:Aw}, we have
\begin{align}
  \|w_{T_1}^\perp\|^2  = (1 - \|Aw_{T_1}\|^2) &\leq \exp(-\sum_{i=1}^{T_1} \frac{(g^*)^2 - \|Ag_i w_i\|^2}{(g^*)^2 + (g^*)^2 - \|Ag_i w_i\|^2})(1 - \|Aw_0\|^2)\\
      &\leq \exp(-\sum_{i=1}^{T_1} \frac{(g^*)^2 - \|Ag_{T_1} w_{T_1}\|^2}{(g^*)^2 + (g^*)^2 - \|Ag_{T_1} w_{T_1}\|^2})(1 - \|Aw_0\|^2) \label{eq:tighter1}\\
    & \leq \exp(-\frac{\rho\delta_0T_1}{(g^*)^2 + \rho\delta_0})(1 - \|Aw_0\|^2)\label{eq:wperp_decrease}
\end{align}
we have $\|w_{T_1}^\perp\|^2=1-\|Aw_{T_1}\|^2\leq \delta^2$  when
\[T_1= \left(1+\frac{(g^*)^2}{\rho\delta_0} \right)\log\left(\frac{1-\|Aw_0\|^2}{\delta^2}\right). \]

Now we only need to verify condition \eqref{eq:delta}. To see this, notice that we use Lemma \ref{lem:g_update}:
\begin{align*}
      g_{t+1} \leq  & g_t + \frac{\gamma g_t}{2}(\|v_t\|^2 - 1)  \text{ for } t<T_1
\end{align*}

Note $\|v_t\|^2={g^{*2}}/{(g_t^2\|Aw_{t+1}\|^2)}>1$ as  $\|Aw_{t+1}\|<1$ and $g^{*2}>g_t^2$. Thus, even though $g_t$ grows with the rate $\gamma^{(1 )}(\|v_{t-1}\|^2-1)$, we use our choice of $\gamma^{(1)}$ and $g_{t+1}<g^*$. In fact, we set $\gamma^{(1)}$ small such that after $T_1$ there is a gap between $g^*$ and $g_{T_1}$. We let the gap satisfies $(g^*)^2-g_{T_1}^2\geq  \rho \delta_0$:\footnote{Note that one could use $\frac{1}{2}\delta_0$ for the convenience of the proof.} 
 \begin{align}
 g_{T_1+1}^2 \leq \prod_{t=1}^{T_1+1}\left(1 +  \frac{1}{2}\gamma^{(1 )}(\|v_{t-1}\|^2-1)\right)^2 g_1^2
 \overset{(a)}{\leq} \exp\left( \gamma^{(1 )}(\|v_{0}\|^2-1)T_1\right) g_0^2
 \overset{(b)}{\leq}  (g^*)^2-\rho\delta_0 \label{eq:control_g}
 \end{align}
where step $(a)$ due to $g_1=g_0$ and the that 
\begin{align}
    \frac{\|v_t\|^2}{ \|v_{t-1}\|^2} \leq \frac{g_{t-1}^2}{g_t^2}\leq 1\label{eq:mono}
\end{align}
which is due to $g_t^2 \|v_t\|^2 - g_{t-1}^2 \|v_{t-1}\|^2<0$ (Lemma \ref{lem:gv} and $g_t -g_{t-1}\|v_{t-1}\|\leq g^*  -\frac{g^*}{\|Aw_t\|}<0$) 
In step $(b)$ as long as we make sure 
\[  \gamma^{(1 )}\leq \frac{\log \left(
 (g^{*2}-\rho\delta_0)/g_0^2\right)}{(\|v_{0}\|^2-1)T_1}=\frac{\|Ag_0w_0\|^2}{2 T_1((g^*)^2-\|Ag_0w_0\|^2)} \log\left((1-\rho)\frac{g^{*2}}{g_0^2}+\rho \right) \]
which is satisfied by  our choice of $\gamma^{(1 )}$ for fixed  $T_1$, $g_0$, and $\delta_0$.

\textbf{The loss$\|A(g_Tw_T-g^*w^*)\|$.}  By Lemma \ref{lem:v_norm}, 
\begin{align*}
    \|A(g_{t+1}w_{t+1} -g^* w^*)\|^2 = (g_{t+1}-g_t\|v_t\|)^2
\end{align*}
which means we only need to analyze the term
\begin{align}
g_t \|v_t\| - g_{t+1} &= (1 - \gamma^{(1 )})(g_{t-1} \|v_{t-1}\| - g_t)\label{eq:g}\\
&-\left( g_{t-1} \|v_{t-1}\|-g_t\|v_t\| \right)\left( 1-\frac{\gamma^{(1 )}}{g_t}\frac{g_{t-1}\|v_{t-1}\|}{(g_t + g_{t-1}\|v_{t-1}\|)}(g_t \|v_t\|+g_{t-1}\|v_{t-1}\| )\right).\notag
\end{align}

Again, for $t<T_{1}$, we have from \eqref{eq:mono} that $g_{t-1} \|v_{t-1}\|-g_t\|v_t\|>0$. Meanwhile, with Lemma \ref{lem:g_increase} and Lemma \ref{lem:gv}, we have
\begin{align*}
    \frac{g_t (g_t + g_{t-1}\|v_{t-1}\|)}{g_{t-1}\|v_{t-1}\| (g_t \|v_t\|+g_{t-1}\|v_{t-1}\|)} &\geq \frac{g_t^2}{g_{t-1}\|v_{t-1}\| (g_t \|v_t\|+g_{t-1}\|v_{t-1}\|)} \\
    &\geq \frac{g_0^2}{2g_0^2\|v_0\|^2} \\
    &\geq \frac{1}{2(1+\|Aw_0\|^2)}
\end{align*}
By our choice of $\gamma^{(1 )}$, \[\gamma^{(1 )}\leq \frac{1}{2\|v_0\|^2}\leq \frac{g_t (g_t + g_{t-1}\|v_{t-1}\|)}{g_{t-1}\|v_{t-1}\| (g_t \|v_t\|+g_{t-1}\|v_{t-1}\|)}\]
We have from \eqref{eq:g} that $t>T_1$
\begin{align}
 g_{t}\|v_t\| - g_{t+1} &\leq (1 - \gamma^{(2)})^{t-T_1}(g_{T_1}\|v_{T_1}\| - g_{T_1}) \nonumber\\
  &= (1 - \gamma^{(2)})^{t-T_1}(g_{T_1}^2\|v_{T_1}\|^2 - g_{T_1}^2) /(g_{T_1}\|v_{T_1}\|+ g_{T_1}) \nonumber\\
&  =(1 - \gamma^{(2)})^{t-T_1}\left( {g^{*2}}/{\|Aw_{T_1}\|^2} -g_{T_1}^2\right)/(g^{*}/{\|Aw_{T_1}\|} +g_{T_1})\\
& \leq (1 - \gamma^{(2)})^{t-T_1}\left(g^{*2} -g_{T_1}^2\right)/\left(g^*(1-\ep)+g_{T_1}\right)\\
&= (1 - \gamma^{(2)})^{t-T_1}\rho \delta_0/\left(g^*(1-\ep)+g_{T_1}\right) \\
&\leq (1 - \gamma^{(2)})^{t-T_1}\rho \delta_0/\left(g^*(1-\ep)\right). 
 \end{align}

So we have $ g_{T}\|v_{T}\| - g_{T} \leq g^*\sqrt{\ep}$ after
$T= T_1+\frac{1}{\gamma^{(2)}}\log \left(\frac{\rho \delta_0}{\left(g^{*2}(1-\ep)\right)\sqrt{\ep}}\right)$

\end{proof}
\subsubsection{Technical Lemmas for Theorem \ref{partI}}

In the following section, we assume $AA^\top = I_{m\times m}$ and use $r_t = A(g_t w_t - g^* w^*)$ to denote the negative residual.

\begin{lemma}\label{lem:v_norm}
With the step-size $\eta=\frac{1}{g_t^2}$, we have the following equalities:
 We have the following property:
\begin{align}
 Aw_{t+1} & = \frac{Ag^* w^*}{g_t\|v_t\|} \label{eq:w}\\
g_t^2\|v_t\|^2=g_t^2+&((g^*)^2- \|Ag_tw_{t}\|^2) =\frac{g^{*2} }{\|Aw_{t+1}\|^2}\label{eq:v}\\
\|Aw_{t+1}\|^2 &=\frac{(g^*)^2}{(g^*)^2+g_t^2(1-\|Aw_t\|^2) } \label{eq:1-Aw} 
\\
1-\|Aw_{t+1}\|^2 &= \frac{1}{\|v_t\|^2}\left(1-\|Aw_{t}\|^2 \right)  \label{eq:g-gAw1}\\
g_{t+1}^2-\|Ag_{t+1}w_{t+1}\|^2 &= \frac{\|Ag_{t+1}w_{t+1}\|^2}{(g^*)^2}\left(g_{t}^2-\|Ag_{t}w_{t}\|^2 \right).  \label{eq:g-gAw}
\end{align}
\end{lemma}
\begin{proof}

With the update of $w_t$:
\begin{align*}
    v_{t+1} = w_t - \eta g_t A^\top r_t, \qquad w_{t+1} = \frac{v_t}{\|v_t\|},
\end{align*}
we can get 
\begin{align*}
    Aw_{t+1} = \frac{Aw_t - \frac{1}{g_t} r_t}{\|v_t\|} = \frac{Ag^* w^*}{g_t\|v_t\|},
\end{align*}
and 
\begin{align*}
     \|v_{t}\|^2 = & \|w_t\|^2 - 2\eta \langle Ag_t w_t, r_t\rangle + \eta^2 g_t^2 \|r_t\|^2 \\
     = & 1 + \eta(\|Ag^* w^*\|^2 - \|Ag_tw_t\|^2 - \|r_t\|^2) + \eta^2 g_t^2\|r_t\|^2
\end{align*}
Moreover,
\begin{align*}
 \|Aw_{t+1}\|^2 &= \frac{1}{\|v_t\|^2}\|Aw_{t} - \eta g_t r_t \|^2  \\
\Leftrightarrow {\|v_t\|^2}\|Aw_{t+1}\|^2 &= \|Aw_{t}\|^2 - 2\eta\langle Ag_tw_{t}, r_t \rangle + \eta^2g_t^2\|r_t\|^2\\
&= \|Aw_{t}\|^2 - 2\eta\langle Ag_t w_{t}, r_t\rangle + \eta^2 g_t^2 \|r_t\|^2\\
&=  \|Aw_{t}\|^2 + \eta(\|Ag^* w^*\|^2 - \|Ag_tw_t\|^2 - \|r_t\|^2) + \eta^2 g_t^2\|r_t\|^2.
\end{align*}
Letting $\eta=\frac{1}{g_t^2}$, we have:
 $$ g_t^2\|v_{t}\|^2 =g_t^2 +((g^*)^2 - \|A g_t w_t\|^2)$$
$$ \|v_t\|^2 \|Aw_{t+1}\|^2 = \frac{ (g^*)^2}{g_t^2}.$$ 
From \eqref{eq:v}, we can get \eqref{eq:1-Aw}, and with \eqref{eq:1-Aw}, we can obtain \eqref{eq:g-gAw}  and \eqref{eq:g-gAw1} after some algebra.
\end{proof}

\begin{lemma}\label{lem:g_update} For $\eta = \frac{1}{g_t^2}$ and  $\gamma = \gamma_t$, we have
 $g_{t+1} \leq   g_t + \frac{\gamma g_t}{2}(\|v_t\|^2 - 1)$ and
\begin{align*}
g_t \|v_t\| - g_{t+1} 
 = &(1 - \gamma)(g_{t-1} \|v_{t-1}\| - g_t)\\
 &-  \left( g_{t-1} \|v_{t-1}\|-g_t\|v_t\| \right)\left( 1-\frac{\gamma}{g_t}\frac{g_{t-1}\|v_{t-1}\|}{(g_t + g_{t-1}\|v_{t-1}\|)}(g_t \|v_t\|+g_{t-1}\|v_{t-1}\| )\right)    
\end{align*}

\end{lemma}
\begin{proof}
The update of $g_{t+1}$ is
\begin{align*}
    g_{t+1}  = & g_{t} - \gamma \langle Aw_t, r_t\rangle\\
    = & g_t + \frac{\gamma}{2g_t} ((g^*)^2 - \|Ag_t w_t\|^2 - \|r_t\|^2) \\
    = & g_t + \frac{\gamma g_t}{2}(\|v_t\|^2 - 1) - \frac{\gamma}{2g_t}\|r_t\|^2.
\end{align*}
where the second equality due to $\langle Ag_tw_t, r_t\rangle=\|Ag^* w^*\|^2 - \|Ag_tw_t\|^2 - \|r_t\|^2$ and the last equality due to update of $\|v_t\|$ (see equality \ref{eq:v}). This finishs the proof for the first inequality.

Denoting $C_t= \alpha_{t-1}^2 g_{t-1}^2 (\|Aw_{t-1}\|^2 - 1)$, we get
\begin{align*}
    g_{t+1} = & g_t + \frac{\gamma}{2g_t}g_t^2(\|v_t\|^2 - 1) - \frac{\gamma}{2g_t}\|r_t\|^2\\
    = & g_t + \frac{\gamma g_t}{2}\|v_t\|^2 - \frac{\gamma g_t}{2}- \frac{\gamma}{2g_t}\left((g_{t-1} \|v_{t-1}\|-g_{t})^2 + C_t\right)\\
        = & g_t + \frac{\gamma g_t}{2}\|v_t\|^2 - \frac{\gamma g_t}{2}- \frac{\gamma}{2g_t}\left((g_{t-1}^2 \|v_{t-1}\|^2+g_{t}^2- 2g_tg_{t-1}\|v_{t-1}\|)+C_t \right)\\
        = & g_t + \frac{\gamma g_t}{2}\|v_t\|^2 - \frac{\gamma}{2g_t}g_{t-1}^2 \|v_{t-1}\|^2+\gamma( g_{t-1}\|v_{t-1}\|-g_{t})-\frac{\gamma}{2g_t}C_t\\        
        = & g_t + \frac{\gamma }{2g_t} \left(g_{t}^2\|v_t\|^2 - g_{t-1}^2 \|v_{t-1}\|^2 \right)+\gamma( g_{t-1}\|v_{t-1}\|-g_{t})-\frac{\gamma}{2g_t}C_t\\    
\Rightarrow g_t\|v_t\|-g_{t+1} = &    g_t\|v_t\|-g_{t}- \frac{\gamma }{2g_t} \left(g_{t}^2\|v_t\|^2 - g_{t-1}^2 \|v_{t-1}\|^2 \right)-\gamma( g_{t-1}\|v_{t-1}\|-g_{t})+\frac{\gamma}{2g_t}C_t\\ 
= &  (1-\gamma)( g_{t-1}\|v_{t-1}\|-g_{t}) -g_{t-1} \|v_{t-1}\|+ g_t\|v_t\|\\
&\underbrace{- \frac{\gamma }{2g_t} \left(g_{t}^2\|v_t\|^2 - g_{t-1}^2 \|v_{t-1}\|^2 \right)+ \frac{\gamma}{2g_t}\alpha_{t-1}^2 g_{t-1}^2 (\|Aw_{t-1}\|^2 - 1)}_{Term1}.
\end{align*}
We prove the lemma with following simplification for $Term1$:
\begin{align*}
  Term1 = & - \frac{\gamma }{2g_t} \left(g_{t}^2\|v_t\|^2 - g_{t-1}^2 \|v_{t-1}\|^2 \right)+\frac{\gamma}{2g_t}\alpha_{t-1}^2 g_{t-1}^2 (\|Aw_{t-1}\|^2 - 1)\\
    = & -\frac{\gamma}{2g_t} (\frac{g_t^2}{\|v_{t-1}\|^2} - g_{t-1}^2)(1 - \|Aw_{t-1}\|^2) + \frac{\gamma}{2g_t}\alpha_{t-1}^2 g_{t-1}^2 (\|Aw_{t-1}\|^2 - 1)\\
    = & -\frac{\gamma}{2g_t}(1 - \|Aw_{t-1}\|^2)(\frac{g_t^2}{\|v_{t-1}\|^2} - g_{t-1}^2 - (\frac{g_t}{\|v_{t-1}\|} - g_{t-1})^2)\\
    = & -\frac{\gamma}{2g_t}(1 - \|Aw_{t-1}\|^2)(\frac{2g_t g_{t-1}}{\|v_{t-1}\|} - 2g_{t-1}^2)\\
    = & -\frac{\gamma g_{t-1}}{g_t}(1 - \|Aw_{t-1}\|^2)(\frac{g_t}{\|v_{t-1}\|} - g_{t-1})\\
    = & -\frac{\gamma g_{t-1}\|v_{t-1}\|}{g_t} \frac{g_t^2\|v_t\|^2 - g_{t-1}^2 \|v_{t-1}\|^2}{g_t + g_{t-1}\|v_{t-1}\|}
\end{align*}
where at the last step we use  Lemma \ref{lem:gv} to have:
\begin{align*}
    (1 - \|Aw_{t-1}\|^2)(\frac{g_t}{\|v_{t-1}\|} - g_{t-1})&= \frac{1}{\frac{g_t}{\|v_{t-1}\|} + g_{t-1}} (1 - \|Aw_{t-1}\|^2)(\frac{g_t^2}{\|v_{t-1}\|^2} - g_{t-1}^2) \\
  &= \frac{g_t^2\|v_t\|^2 - g_{t-1}^2 \|v_{t-1}\|^2}{\frac{g_t}{\|v_{t-1}\|} + g_{t-1}}  \\  
   &= \|v_{t-1}\|\frac{g_t^2\|v_t\|^2 - g_{t-1}^2 \|v_{t-1}\|^2}{g_t+ g_{t-1}\|v_{t-1}\|}  \\   
\end{align*}
\end{proof}
\begin{lemma}\label{lem:g_increase} If $\eta = \frac{1}{g_t^2}$ and $g_t -g_{t-1}\|v_{t-1}\|<0$, we always have that
\begin{align*}
    g_{t+1} > g_t, \quad \forall t
\end{align*}
\end{lemma}
\begin{proof}
Notice that $g_{t+1} = g_t - \gamma \langle Aw_t, r_t\rangle$, so we only need to prove that $\langle Aw_t, r_t \rangle < 0$. Indeed,
\begin{align*}
    \langle Aw_{t}, r_{t}\rangle = g_{t} \|Aw_{t}\|^2 - g^* \langle Aw_t, Aw^* \rangle =\frac{g_t (g^*)^2}{g_{t-1}^2  \|v_{t-1}\|^2} - \frac{(g^*)^2}{g_{t-1} \|v_{t-1}\|} = \frac{(g^*)^2(g_t -g_{t-1}\|v_{t-1}\|)}{g_{t-1}^2\|v_{t-1}\|^2} < 0.
\end{align*}
\end{proof}
\begin{lemma}\label{lem:gv} We have the following identity for the recursion on $g_t^2 \|v_t\|^2$:
$$g_t^2 \|v_t\|^2 - g_{t-1}^2 \|v_{t-1}\|^2=(\frac{g_t^2}{\|v_{t-1}\|^2} - g_{t-1}^2)(1 - \|Aw_{t-1}\|^2) =(g_t^2 - g_{t-1}^2 \|v_{t-1}\|^2)(1 - \|Aw_t\|)^2 .$$
\end{lemma}
\begin{proof}
By Lemma \ref{lem:v_norm}, we use the \eqref{eq:v} to get
\begin{align*}
    g_t^2\|v_t\|^2 - g_{t-1}^2 \|v_{t-1}\|^2 = & g_t^2(1 - \|Aw_t\|^2) - g_{t-1}^2 (1 - \|Aw_{t-1}\|^2)\\
    = & (\frac{g_t^2}{\|v_{t-1}\|^2} - g_{t-1}^2)(1 - \|Aw_{t-1}\|^2) \\
    = & (g_t^2 - g_{t-1}^2 \|v_{t-1}\|^2)(1 - \|Aw_t\|)^2.
\end{align*}
\end{proof}

\begin{lemma}\label{lem:Aw} We have the following bound on the closeness of $Aw_t$ to unit norm:
\begin{align}
   \|w_t^\perp\| \leq (1 - \|Aw_t\|^2) \leq \exp(-\sum_{i=1}^t \frac{(g^*)^2 - \|Ag_i w_i\|^2}{(g^*)^2 + (g^*)^2 - \|Ag_i w_i\|^2})(1 - \|Aw_0\|^2).
\end{align}
\end{lemma}
\begin{proof}



If we keep $g_t \leq g^*$, by  \eqref{eq:1-Aw} we always have that
\begin{align*}
    1 - \|Aw_{t+1}\|^2 = & \frac{g_t^2(1 - \|Aw_{t}\|^2)}{(g^*)^2 + g_t^2 (1 - \|Aw_{t}\|^2)} \\
    \leq & \frac{(g^*)^2}{(g^*)^2 + (g^*)^2 - \|Ag_t w_t\|^2}(1 - \|Aw_{t}\|^2) \\
    \leq & \exp(-\frac{(g^*)^2 - \|Ag_t w_t\|^2}{(g^*)^2 + (g^*)^2 - \|Ag_t w_t\|^2})(1 - \|Aw_{0}\|^2).
\end{align*}
The first inequality holds due to 
\begin{align*}
    \frac{g_t^2}{(g^*)^2 + g_t^2 (1 - \|Aw_{t}\|^2)} \leq \frac{g_t^2 + ((g^*)^2 - g_t^2)}{(g^*)^2 + g_t^2 (1 - \|Aw_{t}\|^2) + ((g^*)^2 - g_t^2)}
\end{align*}

Thus,
\begin{align*}
    (1 - \|Aw_t\|^2) \leq \exp(-\sum_{i=1}^t \frac{(g^*)^2 - \|Ag_i w_i\|^2}{(g^*)^2 + (g^*)^2 - \|Ag_i w_i\|^2})(1 - \|Aw_0\|^2).
\end{align*}
\end{proof}

\subsection{Proof of Theorem \ref{thm:orthogonal} Case (b) }\label{sec:proof_orthogonal1}
Here we discuss the case (b) of Theorem \ref{thm:orthogonal1}
\begin{theorem}[Fixing $g$ in Phase I]
Suppose the initialization satisfies $0< g_0 < g^*$, and that $w_0$ is a random vector with $\|w_0\|=1$. Set $\eta_t=1/g_t^2$ at all iterations. For any $0<\ep<0.5$, let the learning rate of $g$ in Phase II satisfies
\begin{equation}
    0<\gamma< \frac{g^*-g_0}{(1-\ep^2)(g^*-g_0)+\ep^2g^*}<1. \label{eqn:gamma_ub}
\end{equation}
Let the number of iterations be 
\begin{equation}
    T_1 = \frac{\log(1/\ep^2)}{\log(g^{*2}/g_0^2)},\quad T_2 = \frac{\log(\frac{1-(1-\ep^2)g_0/g^*}{\ep^2})}{\log(\frac{1}{1-(1-\ep^2)\gamma})}.
\end{equation}
Then after $T=T_1+T_2$ iterations, the output of Algorithm~\ref{alg:main} will satisfy
\begin{equation}
    \inprod{w_T, w^*} \ge 1-\ep,\;\;  (1-2\ep^2)g^* \le g_T \le g^*,
\end{equation}
which indicates that $g_Tw_T$ is close to the minimum $\ell_2$-norm solution $g^*w^*$. We can also bound the final loss as $f(w_T, g_T) = \|Ag_Tw_T-Ag^*w^*\|^2/2 \le 3\ep g^{*2}$.\label{thm:orthogonal}
\end{theorem}

Simplification for $T_1$ and $T_2$ (here we assume $\gamma < \frac{1}{4}$ to get the order in Theorem \ref{thm:orthogonal1}):
\begin{align}
    T_1& = \frac{\log(1/\ep^2)}{\log(g^{*2}/g_0^2)}= \frac{\log(1/\ep^2)}{\log(1+(g^{*2}-g_0^2)/g_0^2)} \overset{(a)}{\approx}\frac{\log(1/\ep^2)}{(g^{*2}-g_0^2)/g_0^2} = \frac{g_0^2}{\delta_0}\log(1/\ep^2),\\
    T_2& =\frac{\log\left((g^*-g_0)/(g^*\ep^2) + g_0/g^*  \right) }{\log\left(1+(1-\ep^2)\gamma/(1-(1-\ep^2)\gamma) \right)} \overset{(b)}{\approx} \frac{\log\left((g^*-g_0)/(g^*\ep^2) \right)}{(1-\ep^2)\gamma/(1-(1-\ep^2)\gamma) } \overset{(c)}{\approx} \frac{1}{\gamma}\log \left( \frac{g^{*2}-g_0^2}{g^{*2}\ep^2}\right).
\end{align}
For step $(a)$ and $(b)$, we apply $\log(1+x)\geq x \log 2 , x<1$ for denominator. 
For step $(b)$, we take out the constant term $g_0/g^*$ in the numerator. For step $(c)$ inside the $log$ term, we multiply $g^*+g_0$ for both numerator and denominator as follows
\[\frac{g^*-g_0}{g^*\ep^2} = \frac{(g^*-g_0)(g^*+g_0)}{g^*(g^*+g_0)\ep^2}\leq\frac{(g^{*2}-g_0^2)}{g^{*2}\ep^2}. \]
\begin{proof}
For any vector $w\in\R^d$, we use $w^{\parallel}\in\R^d$ to denote its projection onto the row space of $A$. We use $w^{\perp}\in\R^d$ to denote its component in the subspace that is orthogonal to the row space of $A$. Since $A$ has orthogonal rows, we can write $w = w^{\parallel} +  w^{\perp}$, where
\begin{equation}
    w^{\parallel} = A^\top A w,\;\; w^{\perp}= (I - A^\top A)w.
\end{equation}
Since $w^*$ is the minimum $\ell_2$-norm solution, $w^{*\perp}$ must be zero, i.e., $(I-A^\top A)w^*=0$ and $A^\top Aw^* = w^*$.

We will show that the algorithm has two phases. 
We now look at each phase in more detail.

{\bf Phase I.} For any $t=0,...,T_1-1$, only $w$ is updated.
\begin{align}
    v_t
    & \overset{(a)}{=} w_t - \eta_t g_t^2 A^\top Aw_t + \eta_t g_tg^*A^\top Aw^* \nonumber \\
    & \overset{(b)}{=} (I-A^\top A)w_t + \frac{g^*}{g_0}A^\top Aw^* \nonumber\\
    & \overset{(c)}{=} w_t^{\perp} + \frac{g^*}{g_0}w^*, \label{eqn:vt}
\end{align}
where (a) follows from substituting the partial gradient, (b) is true because of the choice of our learning rates: $\eta_t = 1/g_t^2$ and $\gamma_t=0$, and (c) follows from the fact that $A$ has orthonormal rows. Since $w_t^{\perp}$ is orthogonal to $w^*$ and $g_0<g^*$, we have 
\begin{equation}
    \|v_t\|^2 = \|w_t^{\perp}\|^2 + g^{*2}/g_0^2 \ge  g^{*2}/g_0^2>1. \label{eqn:vt_norm}
\end{equation}
After normalization, we have $w_{t+1} = v_t/\|v_t\|$.
As shown in \eqref{eqn:vt}, gradient update does not\footnote{This is always true for linear regression because the gradient $\nabla_w f(w,g)$ lies in the row space of $A$.} change the component in the orthogonal subspace: $v_t^{\perp} = w_t^{\perp}$. Since $\|v_t\|^2>1$, the orthogonal component will shrink after the normalization step:
\begin{equation}
    \|w_{t+1}^{\perp}\|^2 = \frac{\|v_t^{\perp}\|^2}{\|v_t\|^2} = \frac{\|w_t^{\perp}\|^2}{\|w_t^{\perp}\|^2 + g^{*2}/g_0^2} \le \frac{g_0^2}{g^{*2}}\|w_t^{\perp}\|^2. \label{eqn:shrink}
\end{equation}
Since $g_0<g^*$, after $T_1 = \frac{\log(1/\ep^2)}{\log(g^{*2}/g_0^2)}$ iterations, we have
\begin{equation}
    \|w_{T_1}^{\perp}\|^2 \le (g_0^2/g^{*2})^{T_1} \le \ep^2.
\end{equation}
As indicated in (\ref{eqn:vt}), $w_{t}^{\parallel}$ is in the same direction as $w^*$ for $t\ge 1$. Since $\|w_{T_1}^{\perp}\|\le \ep$, $\|w_{T_1}^{\parallel}\| \ge \sqrt{1-\ep^2} \ge 1-\ep$. Therefore, $\inprod{w_{T_1}, w^*} = \|w_{T_1}^{\parallel}\| \ge 1-\ep$.

{\bf Phase II.} For iteration $t=T_1, ..., T_1+T_2-1$, the algorithm updates both $w$ and $g$. The learning rate of updating $g$ is set as a constant $0<\gamma<1$. The gradient update on $g$ is
\begin{align}
    g_{t+1} 
    & = g_t - \gamma g_tw_t^TA^\top Aw_t + \gamma g^*w_t^TA^\top Aw^* \nonumber\\
    & \overset{(a)}{=} g_t - \gamma g_t \|w_t^{\parallel}\|^2 + \gamma g^*\inprod{w_t^{\parallel}, w^*},\nonumber\\
    & \overset{(b)}{=} g_t - \gamma g_t \|w_t^{\parallel}\|^2 + \gamma g^*\|w_t^{\parallel}\|,\label{eqn:gt}
\end{align}
where (a) follows from the fact that $A$ has orthonormal rows and $w^*$ lies in the row space of $A$, and (b) is true because (\ref{eqn:vt}) implies that $w_t^{\parallel}$ is in the same direction as $w^*$ for $t\ge 1$.

We will now prove that the following two properties (see Lemma \ref{lem:fixg}) hold during Phase II:
\begin{itemize}
    \item Property (i): $\|w_{t+1}^{\perp}\| \le \|w_t^{\perp}\| \le \ep$.
    \item Property (ii): letting $\gamma'=\gamma(1-\ep^2)$, we have
    \begin{equation*}
        (1-\gamma')g_t + \gamma'g^*\le g_{t+1} \le g^*.
    \end{equation*}
\end{itemize}
 We will now finish the proof of Theorem~\ref{thm:orthogonal} using these two properties. After $T =T_1+T_2$ iterations, by Property (i) and the same argument as in Phase I, we have $\inprod{w_{T}, w^*} = \|w_T^{\parallel}\|\ge 1-\ep$. By Property (ii), we can rewrite the lower bound of $g_{T}$ as
\begin{align}
    g^* - g_{T} &\le (1-\gamma')(g^* - g_{T-1}) \nonumber \\
    &\le (1-\gamma')^{T_2}(g^*-g_{T_1}) \nonumber \\
    &\overset{(a)}{ =} (1-\gamma')^{T_2}(g^*-g_0) \nonumber \\
    &\overset{(b)}{\le} 2\ep^2 g^*, \label{eqn:g_T}
\end{align}
where (a) follows from the fact that $g_{T_1}=g_0$, and (b) follows from our choice of $T_2$: it is easy to verify that $T_2$ satisfies $(1-\gamma')^{T_2}(g^*-g_0+\delta)=\delta$, which implies that $(1-\gamma')^{T_2}(g^*-g_0)<\delta$. By our definition, $\delta= \ep^2g^*/(1-\ep^2)<2\ep^2 g^*$ for $0<\ep<0.5$. Therefore, by \eqref{eqn:g_T}, we have $g_T\ge (1-2\ep^2)g^*$.

Given $\inprod{w_T, w^*} \ge 1-\ep$ and $(1-2\ep^2)g^*\le g_T \le g^*$, we can bound the loss as
\begin{align*}
    f(w_T, g_T)
    & = g_T^2\|Aw_T\|^2/2 - g_Tg^*\inprod{Aw_T, Aw^*} + g^{*2}/2\nonumber\\
    & \le g^{*2}/2 - (1-2\ep^2)g^{*2}(1-\ep) +g^{*2}/2 \nonumber \\
    & \le 3\ep g^{*2}.
\end{align*}
\end{proof}

\subsubsection{Technical Lemmas for Theorem \ref{thm:orthogonal}}
\begin{lemma}\label{lem:fixg}
We have following property in Phase II for Theorem \ref{thm:orthogonal}
\begin{itemize}
    \item Property (i): $\|w_{t+1}^{\perp}\| \le \|w_t^{\perp}\| \le \ep$.
    \item Property (ii): letting $\gamma'=\gamma(1-\ep^2)$, we have
    \begin{equation*}
        (1-\gamma')g_t + \gamma'g^*\le g_{t+1} \le g^*.
    \end{equation*}
\end{itemize}
\end{lemma}
\begin{proof}
We will argue by induction.  
We first show that the above two properties hold when $t=T_1$. Since $g_{T_1} = g_0<g^*$, by (\ref{eqn:shrink}), we have $\|w_{T_1+1}^{\perp}\| \le g_{T_1}\|w_{T_1}^{\perp}\|/g^{*}< \|w_{T_1}^{\perp}\|\le \ep$. By (\ref{eqn:gt}), we have
\begin{align}
    g_{T_1+1}& \overset{(a)}{\le} g_{T_1} - \gamma g_{T_1}(1-\ep^2) + \gamma g^* \nonumber\\
    & = g_{T_1} - \gamma g_{T_1}(1-\ep^2) + \gamma(1-\ep^2) g^* + \gamma \ep^2 g^* \nonumber\\
    & \overset{(b)}{=} g_{0} - \gamma'g_0 + \gamma'g^* + \gamma'\delta\nonumber\\
    & = g^* - (1-\gamma') (g^*-g_0) + \gamma'\delta, \label{eqn:g_T1_ub}
\end{align}
and
\begin{align}
    g_{T_1+1}& \overset{(c)}{\ge} g_{T_1} - \gamma g_{T_1}\|w_{T_1}^{\parallel}\|^2 + \gamma g^*\|w_{T_1}^{\parallel}\|^2 \nonumber\\
    & = g_{T_1} + \gamma (g^*-g_{T_1})\|w_{T_1}^{\parallel}\|^2 \nonumber\\
    & \overset{(d)}{\ge} g_{T_1} + \gamma (g^*-g_{T_1})(1-\ep^2) \nonumber\\
    & = (1-\gamma')g_{T_1} + \gamma'g_{T_1}, \label{eqn:g_T1_lb}
\end{align}
where inequalities (a), (c), and (d) follow from the fact that $1-\ep^2\le \|w_{T_1}^{\parallel}\|^2 \le 1$, and (b) follows from our definition $\gamma'=\gamma(1-\ep^2)$, $\delta = \ep^2g^*/(1-\ep^2)$, and the fact that $g_{T_1}=g_0<g^*$. By the upper bound of $\gamma$ given in (\ref{eqn:gamma_ub}), we can verify that $(1-\gamma')(g^*-g_0)>\gamma'\delta$, and hence, (\ref{eqn:g_T1_ub}) implies that $g_{T_1+1}< g^*$. 

Now suppose that Property (i) and (ii) hold for $t=T_1, ..., k-1$, where $T_1\le k-1< T_1+T_2-1$. We need to prove that they also hold for the $k$-th iteration. By assumption, $g_k\le g^*$, so using the same argument as (\ref{eqn:vt_norm}) and (\ref{eqn:shrink}), we have $\|v_k\|^2 \ge 1$ and $\|w_{k+1}^{\perp}\| = \|w_k^{\perp}\|/\|v_k\| \le \|w_k^{\perp}\| \le \ep$, where the last step follows from Property (i) at the $(k-1)$-th iteration. Therefore, Property (i) holds for the $k$-th iteration. 

To prove Property (ii), first note that by assumption, $1-\ep^2\le \|w_{k}^{\parallel}\|^2\le 1$. We can use the same argument as (\ref{eqn:g_T1_lb}) to show that $g_{k+1} \ge (1-\gamma')g_{k}+\gamma'g^*$. We can also use a similar argument as (\ref{eqn:g_T1_ub}) to get
\begin{equation}
    g_{k+1} \le g^* - (1-\gamma')(g^*-g_k) + \gamma'\delta, 
\end{equation}
where $\gamma'=\gamma(1-\ep^2)$ and $\delta = \ep^2g^*/(1-\ep^2)$. The above equation can be rewritten as
\begin{align}
    g^* - g_{k+1} +\delta &\ge (1-\gamma')(g^*-g_k+\delta) \nonumber\\
    &\ge (1-\gamma')^{k+1-T_1}(g^*-g_{T_1} + \delta) \nonumber\\
    &\overset{(a)}{\ge} (1-\gamma')^{T_2}(g^*-g_0 + \delta)\nonumber\\
    &\overset{(b)}{=} \delta, \label{eqn:g_k_1}
\end{align}
where (a) follows from the fact that $k\le T_1+T_2-1$, and (b) can be verified for our choice of $T_2$. Eq. (\ref{eqn:g_k_1}) implies that $g_{k+1}\le g^*$.

\end{proof}


\subsection{General A matrix}
\label{sec:general}

\begin{proposition}[General Matrix $A$]
\label{thm:non-orthogonal1}
For a full rank matrix $A$ with $\lambda_{\max}(AA^\top)=1$, we fix $\delta>0$. In \textbf{Phase I} with fixed $g=g_0$ that satisfies $g_0\leq\frac{g^*\lambda_{\min}(AA^\top)}{2+\delta}$
    ,we can reach to a solution satisfying  $\|w_{T_1}^\perp\|\leq \varepsilon$  where $$T_1= \frac{1}{\log(1+\delta)} \log\left( \frac{\|w_0^{\perp}\|}{\varepsilon}\right).$$
Moreover, if the singular values of $A$ do not decrease too fast, so that the following inequality holds:
\begin{align}
    \frac{1}{m} \|A A^\top\|_F^2 \geq \lambda^2_{\min}(AA^\top) + 2\sqrt{\frac{\log (m)}{m}},
\end{align}
and $w^*$ is randomly drawn on the sphere, then with probability $1-\mathcal{O}(\frac{1}{m})$,
we only need that
\begin{align}
g_0 \leq \frac{g^* } {(2+\delta)}\sqrt{\frac{\|AA^\top \|_F^2 - 2\sqrt{m\log (m)}}{m}}.
\label{eq:g02}\end{align} 
\end{proposition}

\begin{lemma}[For all $w^*$]\label{lem:anyw}
Let $\sigma_i$ be the singular values of $A$ in decreasing order, let $r$ be the rank of $A$, so that $\sigma_r>0$.
We fix $g:=g_0$ satisfying $$g_0\leq \frac{g^*\sigma_r}{2+\delta-\sigma_r}$$ and  update only $w$ using rPGD. Then we have the orthogonal component $w^{\perp}$ decreases geometrically such that  
$\|w_{T_1}^{\perp}\|\leq \epsilon $
after iteration $$T_1= \frac{1}{\log(1+\delta)} \log\left( \frac{\|w_0^{\perp}\|}{\epsilon}\right)$$
\end{lemma}
\begin{proof}

Consider the singular value decomposition of $A^\top A=U\Sigma U^\top$ with
\begin{align}
\Sigma =
\begin{bmatrix}
    \sigma_{1} & & &\\
    & \ddots & &\\
    & & \sigma_{m}&\\
        & & & \vzero_{d-m}\\
    \end{bmatrix}
\quad 
\text{with } 1=\sigma_1\geq\sigma_2\geq \ldots \geq \sigma_m > 0.     \label{eq: }
\end{align}
Moreover $U$ is a $d\times d$ orthogonal matrix.
We now use superscripts $t$ to illustrate the $t$th iteration 
$w_t$ since we  use subscript for the eigenvalues index.
Let $\eta = \frac{1}{g_t^2\sigma_1}=\frac{1}{g_t^2}$. The update of $v_t$ can be written as
$$v_t = w_t -\eta g_0A^\top A(g_0w_t-g^*w^*)=(I-A^\top A)w_t +\frac{g^*}{g_0}A^\top Aw^*= U\left( I- \Sigma\right)U^\top w_t+\frac{g^*}{g_0}U\Sigma U^\top w^*$$

\begin{align}
   \|v_t\| &=  \| \frac{g^*}{g_0}\Sigma U^\top w^*+\left( I- \Sigma\right)U^\top w_t\| \nonumber \\
   &=   \left\|\frac{g^*}{g_0}\begin{bmatrix}  
   \sigma_{1} & & &\\
    & \ddots & &\\
    & & \sigma_{m}&\\
        & & & \vzero_{d-m}\\
        \end{bmatrix}U^\top w^*+\begin{bmatrix}  
   0 & & &\\
    & \ddots & &\\
    & &0&\\
        & & & \vone_{d-m}\\
        \end{bmatrix} U^\top w_t +\begin{bmatrix}  
   1-\sigma_{1} & & &\\
    & \ddots & &\\
    & & 1-\sigma_{m}&\\
        & & & \vzero_{d-m}\\
        \end{bmatrix} U^\top w_t \right\| \nonumber \\
&\geq \sqrt{\left(\frac{g^*}{g_0}\right)^2\sum_{i=1}^m\sigma_i^2 [U^\top w^*]_i^2+\sum_{i=m+1}^d[U^\top w_t]_i^2}- \sqrt{\sum_{i=1}^m(1-\sigma_i)^2 [U^\top w_t]_i^2} \nonumber \\
&\geq \frac{g^*}{g_0} \sigma_m -  (1-\sigma_m)  \label{eq:normvt} \\
&\geq \left(\frac{g^*}{g_0}+1\right) \sigma_m-1\nonumber \\
&\geq 1+\delta \nonumber
\end{align}
since we have
$$\sigma_m \geq \frac{2+\delta }{ \left(\frac{g^*}{g_0}+1\right) }\quad \Leftrightarrow \quad g_0\leq \frac{g^*\sigma_m}{2+\delta-\sigma_m} \text{ and } \sigma_m\leq2$$
Note that the singular values are sorted so that $\sigma_m \le \sigma_1 = 1$, so the second inequality clearly holds.
The above inequality implies that as long as we have $g_0$ small, we can always guarantee $\|v\|\geq 1+\delta$. Using the equality 
$$\|w^\perp_{t+1}\|=\frac{\|w^\perp_{t}\|}{\|v_{t+1}\|}\leq \frac{1}{1+\delta} \|w^\perp_{t}\|$$
we see that the orthogonal component $w^{\perp}$ decreases geometrically.
\end{proof}
\begin{lemma}[random vector $w^*$ uniformly distributed on the sphere]\label{lem:randomw}
Suppose further that  $w^*$ is randomly drawn on the sphere, i.e. $w^*= \frac{z}{\|z\|}$ where $z\sim \mathcal{N}(0,I_{d})$. If the input data matrix $A$ satisfies :
\begin{itemize}
\item the maximum eigenvalue of $\lambda_{max}(AA^\top)=1$
\item the rank of $AA^\top$ is $m$.
\item the spectral of $A\in\mathbb{R}^{m\times d}$ satisfies $\frac{1}{m}\left(\|A A^\top\|_F^2-2\sqrt{m\log(m)}\right)\geq \sigma_m^2$ where the $\sigma_m$ is the minimum eigenvalue of $AA^\top$.
\end{itemize}

Then we can fix $g:=g_0$ satisfying 
$$g_0 \leq \frac{g^*} {(2+\delta-\sigma_m)} \sqrt{ \frac{\|A A^\top\|_F^2-2\sqrt{m\log(m)}}{m} }$$  
and update only $w$ using rPGD. Then with probability $1 - \mathcal{O}(\frac{1}{m})$, we have the orthogonal component $w^{\perp}$ decreases geometrically such that  
$\|w_{T_1}^{\perp}\|\leq \epsilon $
after iteration $$T_1= \frac{1}{\log(1+\delta)} \log\left( \frac{\|w_0^{\perp}\|}{\epsilon}\right)$$
\end{lemma}
\begin{proof}

Since $w^*$ is uniform on the $d$-dimensional sphere, $u^*:=U^\top w^*$ (let $ [U^\top w^*]_i =u^*_i$) is also uniform on the $d$-dimensional sphere. Moreover, we can represent the random vector $u^*$ as a standard Normal random vector divided by its norm, i.e.
\begin{align*}
    u^*=\frac{z}{\|z\|},
\end{align*}
where $z\sim \mathcal{N}(0,I_{d})$. We want to lower bound $\sum_{i=1}^m\sigma_i^2 [U^\top w^*]_i^2 = \sum_{i=1}^m\sigma_i^2 (u^*_i)^2$. We can write
\begin{align*}
    \sum_{i=1}^m(u^*_i)^2\sigma_i^2 = \frac{\sum_{i}^m z_{i}^2\sigma_i^2}{\sum_{i=1}^m  z_{i}^2}.
\end{align*}
Thus we need to get the upper bound of $\sum_{i=1}^m z_{i}^2$ and the lower bound of $\sum_{i=1}^{m} z_{i}^2\sigma_i^2$. Note that 
$$1=\sigma_1\geq\sigma_2\geq \ldots \sigma_r>0.$$

Since $z_i \sim \mathcal{N}(0, 1)$, $X=\sum_{i=1}^m z_i^2$ is  $1$-sub-exponential r.v. with expectation $1$. Thus, with Bernstein inequality (i.e., see Theorem 2.8.1 in \cite{vershynin2018high}), for $\epsilon > 0$, we have that:
\begin{align}
 \mathbb{P}\left(\sum_{i=1}^m \sigma_i^2 z_i^2 
\leq \sum_{i=1}^m \sigma_i^2 - \epsilon m\right)
&\leq  \mathbb{P}\left( \left|\frac{1}{m}\sum_{i=1}^m \sigma_i^2 z_i^2 - \frac{1}{m}\sum_{i=1}^m \sigma_i^2 \right|\leq \epsilon \right)\nonumber\\
&\leq 2\exp\left(-c\min \left(\frac{\epsilon^2 m^2}{\sum_{i=1}^m \sigma_i^4}, \frac{\epsilon m }{\max_i \sigma_i^2}\right)\right)\nonumber\\\
&\leq 2\exp\left(-c\min \left(\frac{\epsilon^2 m^2}{\sum_{i=1}^m \sigma_i^2}, \frac{\epsilon m }{\max_i \sigma_i^2}\right)\right) \quad \text{ since }\quad \sum_{i=1}^m \sigma_i^2\geq \sum_{i=1}^m \sigma_i^4
\label{eq:hd1}
\end{align}
and
\begin{align}
   \mathbb{P}\left(\sum_{i=1}^m z_i^2 \geq (1+\epsilon)m\right) \leq  \mathbb{P}\left(\left|\frac{1}{m}\sum_{i=1}^m(z_i^2-1) \right|\geq \epsilon\right)\leq   2\exp(-c \min\{ m\epsilon^2,m\epsilon\} ). \label{eq:hd2}
\end{align}
where $c$ is an absolute constant.

Let $\epsilon=\sqrt{\frac{\log(m)}{m}}\leq 1$. Then,
$ \epsilon^2 m\leq \frac{\epsilon^2 m^2}{\sum_{i=1}^m \sigma_i^2} < \frac{\epsilon m}{\max_i \sigma_i^2}=\epsilon m$ since $2\sqrt{m\log m} < \sum_{i} \sigma_i^2$.
Thus \eqref{eq:hd1} and \eqref{eq:hd2} can be simplified, respectively,
\begin{align}
 \mathbb{P}\left(\sum_{i=1}^m \sigma_i^2 z_i^2 
 \leq \sum_{i=1}^m \sigma_i^2 - \sqrt{m\log(m)}\right)
&\leq \exp\left(-cm\epsilon^2\right)=\frac{1}{m}e^{-c}=\mathcal{O}\left( \frac{1}{m}\right)
\end{align}
and
\begin{align}
   \mathbb{P}\left(\sum_{i=1}^m z_i^2 \geq m + \sqrt{m\log(m)}\right)
   \leq \frac{1}{m}e^{-c}=\mathcal{O}\left( \frac{1}{m}\right). 
\end{align}
Then with probability $1-\mathcal{O}\left( \frac{1}{m}\right)$, we have
\begin{align*}
    \sum_{i=1}^m(u^*_i)^2\sigma_i^2 = \frac{\sum_{i}^mz_{i}^2\sigma_i^2}{\sum_{r=1}^m  z_{i}^2} \geq \frac{\sum_{i=1}^m \sigma_i^2 - \sqrt{m\log(m)} }{ m + \sqrt{m \log m}} \geq \frac{1}{m}\sum_{i=1}^m \sigma_i^2 - 2\sqrt{\frac{\log(m)}{m}} \geq \sigma_m^2
\end{align*}
where the last inequality is due to the assumption that the spectral satisfies $\frac{1}{m}\sum_{i=1}^m \sigma_i^2 > 2\sqrt{\frac{\log(m)}{m}}+\sigma_m^2$. 
To sum up, with probability $1-\mathcal{O}\left( \frac{1}{m}\right)$,
\begin{align*}
 \sum_{i=1}^m \sigma_i^2 [{u}^*]_i^2 \geq\frac{1}{m}\sum_{i=1}^m \sigma_i^2 - 2\sqrt{\frac{\log(m)}{m}} = \frac{1}{m}\left(\|A A^\top\|_F^2-2\sqrt{m\log(m)}\right)
\end{align*}
Now, using the derivation in \eqref{eq:normvt} for lower bound of $\|v_t\|$, we have that: 
\begin{align*}
    \|v_t\| &\geq \frac{g^*}{g_0}\sqrt{ \frac{\|A A^\top\|_F^2-2\sqrt{m\log(m)}}{m}}- (1- \sigma_m)  \geq 1+\delta\\
    \Rightarrow g_0 &\leq \frac{g^*} {(2+\delta-\sigma_m)} \sqrt{ \frac{\|A A^\top\|_F^2-2\sqrt{m\log(m)}}{m} }
\end{align*}
With $g_0$ satisfying above, we can guarantee that $\|v_t\| \geq 1+\delta$.

\end{proof}

\section{Experiments}\label{sec:lr_exp}
We evaluate WN and \ourAlgo on two problems: linear regression and matrix sensing. Due to space limit, we only include the experiment for linear regression here and put the experiment for matrix sensing to the appendix and matrix sensing. We show that for a wide range of initialization, WN and \ourAlgo converges to the minimum $\ell_2$-norm solution for linear regression and the minimum nuclear norm solution for matrix sensing. 
This is in contrast to the standard GD algorithm. 
For both problems GD requires initialization very close to, or exactly at, the origin to converge to the minimum norm solution~\cite{li2018algorithmic}. 
We will compare with the following two step-size schemes.

(1) \textbf{Algorithm with $\gamma_t=\eta_t$}: We simultaneously update the weight vector (matrix) and the scalar $g$. This is similar to the training of deep neural networks, where we use the same learning rate for all of the layers.

(2) \textbf{Two-phase algorithm}: In \textbf{Phase I}, we use sufficiently small learning rate to update $g$, the scale component (a scalar in linear regression). In \textbf{Phase II}, we use large step-size to update $g$. For both phases, we use large learning rate to update the direction component (weight vector in linear regression and weight matrix in matrix sensing)

\subsection{Linear Regression}
Let $m=20$, $d=50$. We generate the feature matrix as $A = U\Sigma V^T\in\R^{m\times d}$, where $U\in\R^{m\times m}$ and $V\in\R^{d\times m}$ are two random orthogonal matrices chosen uniformly over the Stiefel manifold of partial orthogonal matrices, and $\Sigma$ is a diagonal matrix described below.  Let $\kappa = \frac{\lambda_{\max}(AA^\top)}{\lambda_{\min}(AA^\top)}$. We vary the condition number $\kappa \ge 1$ of $A$ in our experiments. The diagonal entries of $\Sigma$ are set as $1, (1/\kappa)^{1/(m-1)}, (1/\kappa)^{2/(m-1)},..., 1/\kappa$. Set $g^*=3$, and $w^*$ as an arbitrary unit norm vector.
\begin{figure}[t]
    \centering
    \includegraphics[width=0.75\linewidth]{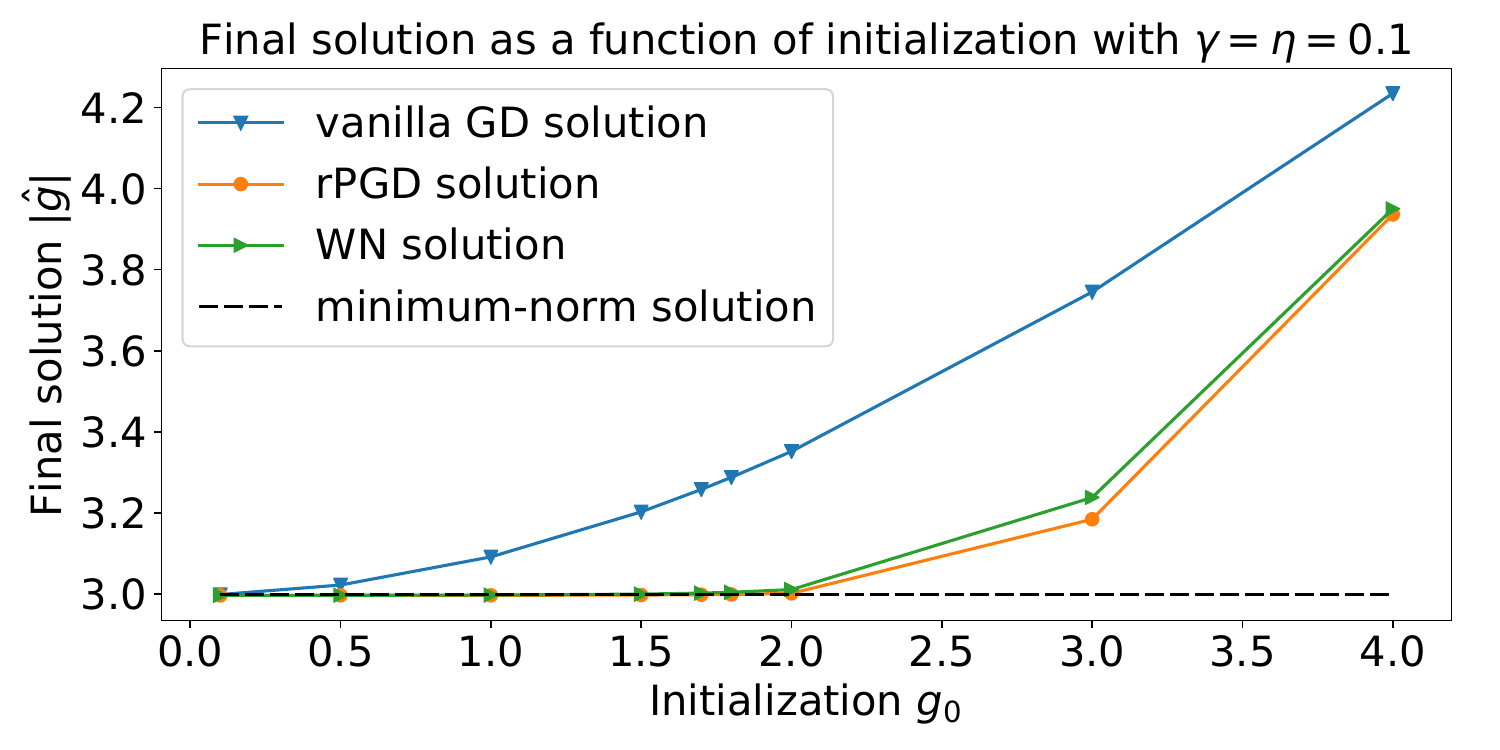}
    \includegraphics[width=0.75\linewidth]{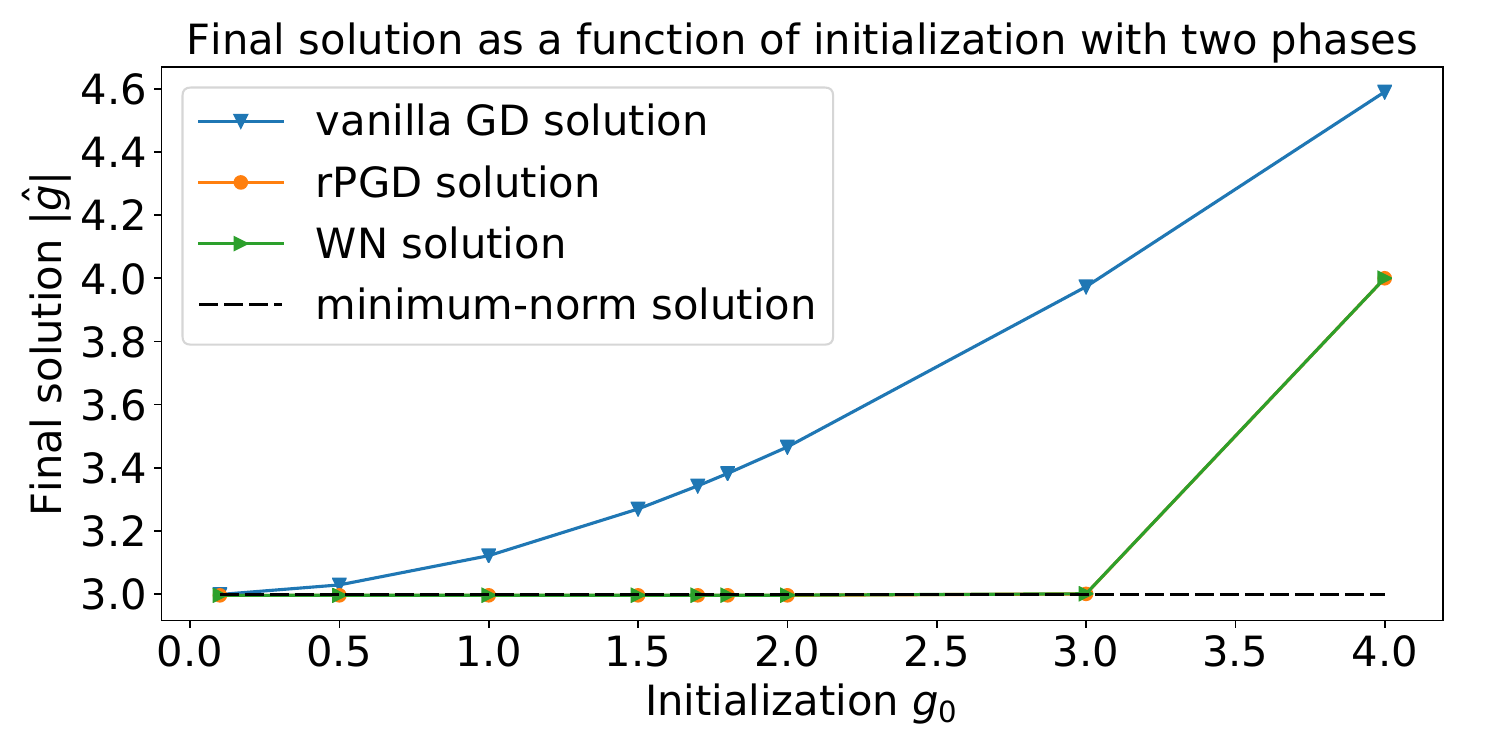}
    \caption{\textbf{Fixed Orthogonal Matrix $A$.} Comparison of the final solutions $|\widehat{g}|=\|\widehat{g}\widehat{w}\|=||\widehat{x}||$ provided by GD, WN and our proposed \ourAlgo on an overparameterized linear regression problem $\min_x\frac{1}{2}\|Ax-y\|_2^2$. All algorithms start from the same initialization $x_0 = g_0w_0$. Compared to GD, WN and \ourAlgo converge to the minimum $\ell_2$-norm solution for a larger region of initialization. Top plot is when we use the same stepsize for $w$ and $g$: $\gamma_t=\eta_t=0.1$. Bottom plot is when we use a particularly small stepsize for $g$ and optimal stepsize for $w$. This implies that a small stepsize for $g$ can arrive to a solution that is close to the minimum-norm solution for even wider range of $g_0$.
    }
    \label{fig:lr_init}
\end{figure}

Let $w_0$ be a random unit norm vector. We run the standard gradient descent (GD) algorithm on the problem 
 \ref{eqn:lsq}
with the initialization  $x_0=g_0w_0$. We run Algorithm~\ref{alg:wn} and \ref{alg:main} starting from the same initialization, and plot $|\widehat{g}|=\|\widehat{g}\widehat{w}\|_2$ as a function of $g_0$. We run all of the algorithms until the squared loss satisfies $f(\widehat{w}, \widehat{g})\le 10^{-5}$, where the final solution is denoted as $\widehat{g}\widehat{w}$.
We have the following observations:

 \textbf{Figure \ref{fig:lr_init_intro}} shows the result when we set a very small but equal learning rate for $w$ and $g$: $\eta_t = \gamma_t=0.005$. It shows there is no difference between WN and rPGD when the learning rate is small, which matches Lemma \ref{cor:flow}. We can see that both WN and rPGD can get close to the minimum norm solution with a large range of initializations ($g^* w^*$ for $g_0\lessapprox 1.5$) while this is only true for GD when $g_0$ is close to $0$. This experiment supports our theory.
 
\textbf{The top plot in Figure \ref{fig:lr_init}} shows the result when we set relatively large learning rates of $w$ and $g$: $\eta_t = \gamma_t=0.1$, as in practice where we use the same non-vanishing learning rate for all the layers when training deep neural networks. The plot shows a difference between WN and rPGD when $g_0>2$, while the two perform similarly when $g_0<2$.
    
 \textbf{The bottom plot in Figure \ref{fig:lr_init}} is when we set (1) WN $\eta_t=\|w_t\|/(g_t^2\lambda_{\max})$ for $w$ and $\gamma_t =0.005 $ for $g$; (2) rPGD $\eta_t = 1/(g_t^2\lambda_{\max})$ and $\gamma_t=0.005$. This mimics the two-phase algorithm as shown in Theorem \ref{thm:orthogonal1}. We can arrive at a solution close to the minimum norm solution for even wider range of $g_0\lessapprox 3$.

\begin{figure}[tb]
    \centering
    \includegraphics[width=0.75\linewidth]{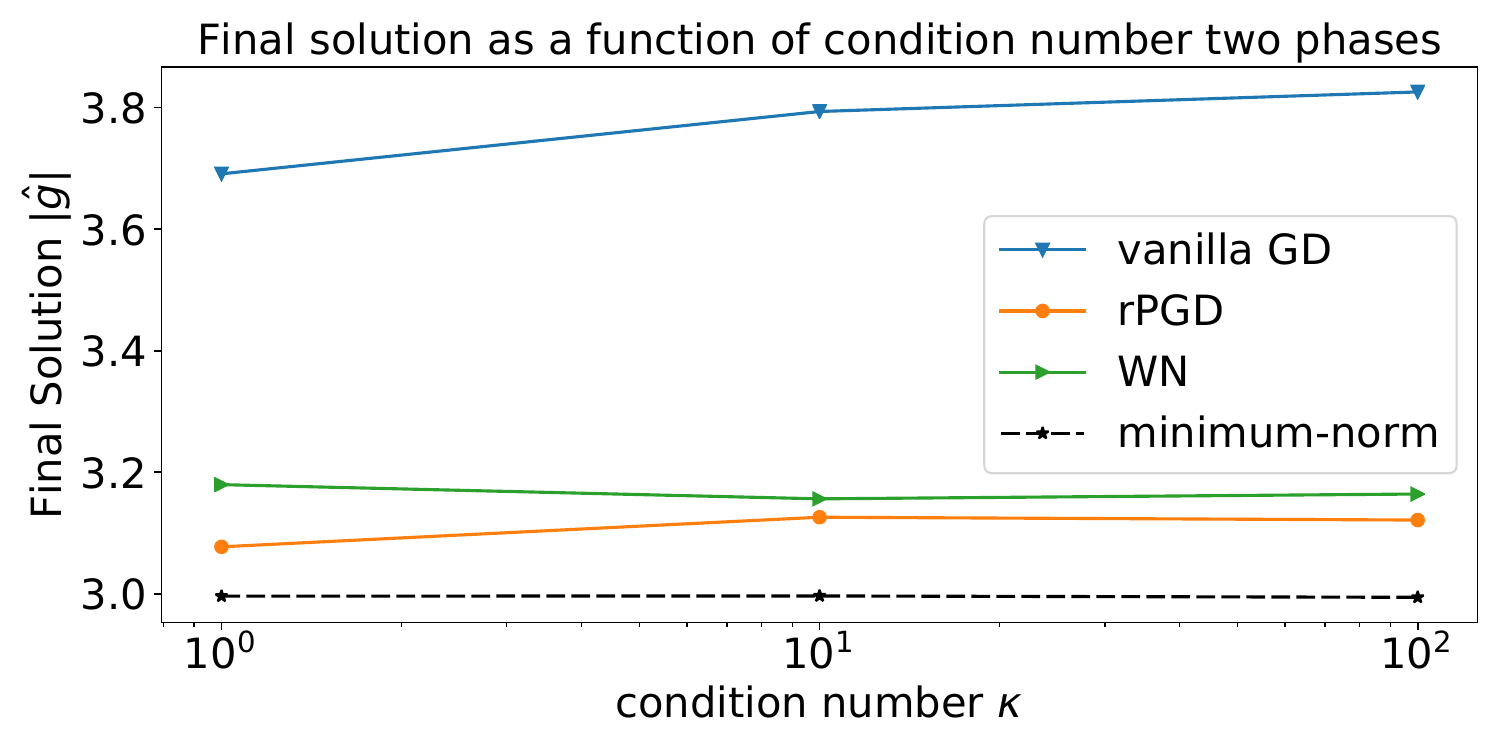}
    \centering
    \includegraphics[width=0.75\linewidth]{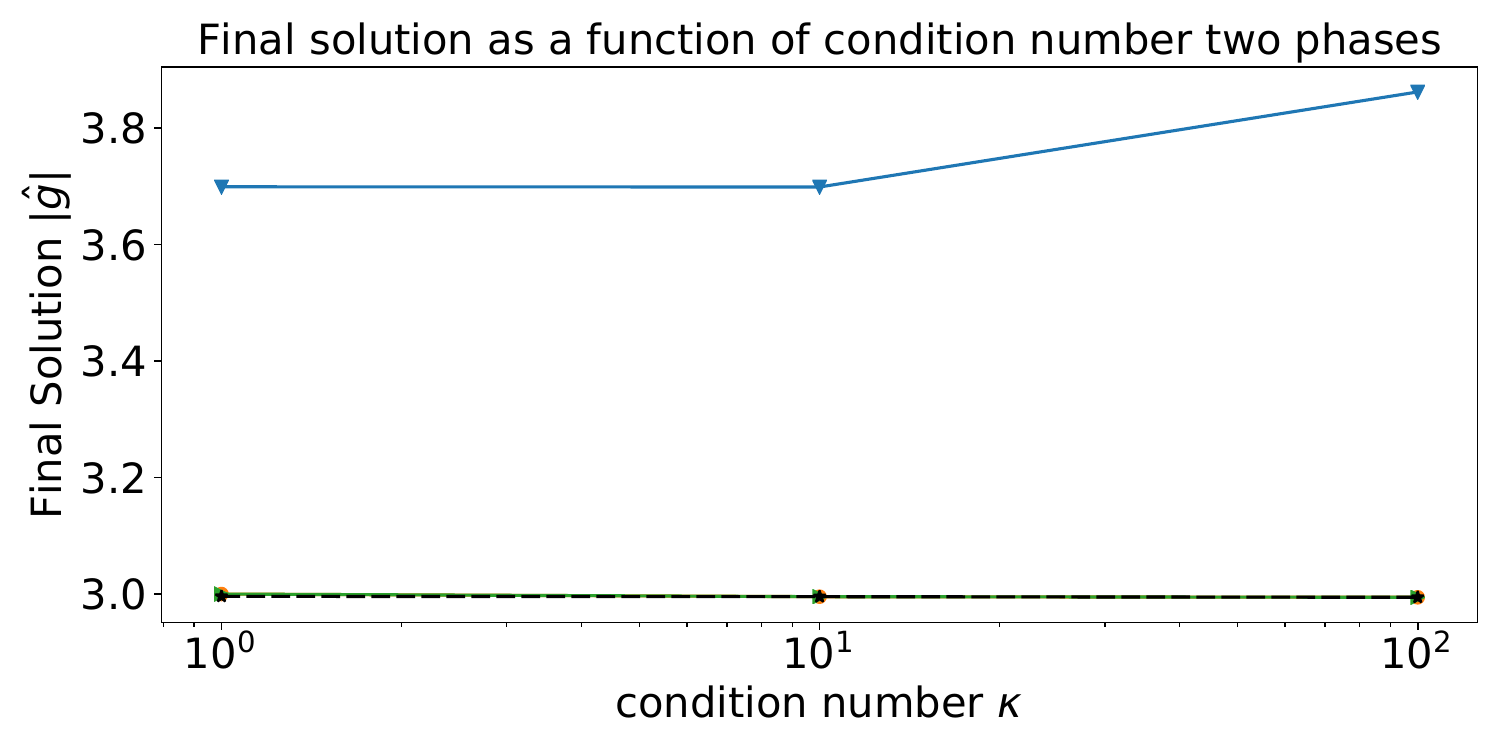}
    \caption{\textbf{Various General Matrix $A$}. Fix $g_0=2.8$ and increase the condition number $\kappa$. Top plot: $\gamma_t=\eta_t=0.01$. The $\ell_2$-norm of the WN and \ourAlgo solutions increases slowly as $\kappa$ increases, but their norm is smaller than when using Gradient Descent. Bottom plot: $\gamma_t=1, \eta_t=0.1\times\mathbb{1}_{\{t>5000\}}$. The $\ell_2$-norm of WN and \ourAlgo solutions are robust to condition number and close to the minimum $\ell_2$-norm for any $\kappa$. 
    Note that green, orange and black curves of the bottom plot overlap. 
    }
    \label{fig:lr_kappa}
\end{figure}
\textbf{Robustness to the condition number $\kappa$.} We repeat the previous experiment for various input matrix $A$ with a  wider range of $\kappa$ with fixed initialization $g_0=2.8$. The top plot in  Figure \ref{fig:lr_kappa} shows that for $\gamma_t=\eta_t$ as $\kappa$ increases, the $\ell_2$-norm of the solutions provided by WN and \ourAlgo also gradually increases but not as much as those provided by the vanilla GD. The bottom plot in Figure \ref{fig:lr_kappa} shows that the performance of the two-phase algorithms, with $\eta_t=0$ in the first 5000 iterations, thus have a better performance compared with algorithm using $\gamma_t=\eta_t$.

\subsection{Experiment: Matrix Sensing}\label{sec:mat_sen}
We show that the normalization methods can also be applied to the matrix sensing problem, to get closer to the minimum nuclear norm solutions. The goal in the matrix sensing problem is to recover a low-rank matrix from a small number of random linear measurements. Here we follow the setup considered in~\cite{li2018algorithmic} 
(for more related work on matrix sensing and completion, see, e.g., \cite{candes2009exact,donoho2013phase,ge2016matrix} and references therein). Let $X^*=U^*U^{^*T}\in\R^{d\times d}$ (with $U^*\in\R^{d\times r}$) be the ground-truth rank-$r$ matrix. Let $A_1,.., A_m\in\R^{d\times d}$ be $m$ random sensing matrices, with each entry sampled from a standard Gaussian distribution. We are interested in the setting when $r\ll d$ and $m\ll d^2$. Given $m$ linear measurements of the form $\inprod{A_i, X^*}$, let $U\in\R^{d\times d}$ be the variable matrix, we define the (over-parameterized) loss function as
\begin{equation}
    f(U) = \frac{1}{2m}\sum_{i=1}^m(\inprod{A_i,UU^T}- \inprod{A_i, X^*})^2. \label{eqn:uut}
\end{equation}
It is proved in~\cite{li2018algorithmic} that if $m=\tilde{O}(dr^2)$, then gradient descent on $f(U)$, when initialized very close to the origin, can recover the low-rank matrix $X^*$.  

\paragraph{WN.}
To apply WN, we need to reparametrize $U$ into a direction variable and a scale variable. We consider two choices:
\begin{itemize}
    \item Let $UU^T = g\frac{WW^T}{\|W\|_F^2}$, where $g\in\R$, and $W\in\R^{d\times d}$. In Figure  \ref{fig:matrix}, the green curve represents this algorithm. We label it with  \textbf{WN}.
    \item Let $UU^T = WDW^T$, where $D\in\R^{d\times d}$ is a diagonal matrix, and all the column vectors of $W\in\R^{d\times d}$ have unit $\ell_2$ norm. That is, for $w_i \in \mathbb{R}^d, i=1,2,\ldots,d$
    $$W=\left[\frac{w_1}{\|w_1\|};\frac{w_2}{\|w_2\|};\ldots;\frac{w_d}{\|w_d\|} \right].$$ In Figure  \ref{fig:matrix}, the purple curve represents the algorithm. We label it with \textbf{WN-Diag} where ``Diag" references the diagonal matrix.
\end{itemize}

\paragraph{rPGD.} To apply \ourAlgo, we need to reparametrize $U$ into a direction variable and a scale variable. We consider two choices:
\begin{itemize}
    \item Let $UU^T = gWW^T$, where $g\in\R$, and $W\in\R^{d\times d}$ satisfies $\|W\|_F=1$. See Algorithm \ref{alg:gwwt}. In Figure  \ref{fig:matrix}, the orange curve represents the algorithm. We label it with  \textbf{rPGD}.
    \item Let $UU^T = WDW^T$, where $D\in\R^{d\times d}$ is a diagonal matrix, and all the column vectors of $W\in\R^{d\times d}$ are projected to have unit $\ell_2$ norm. See Algorithm \ref{alg:wdwt}.  In Figure  \ref{fig:matrix}, the red curve represents the algorithm. We label it with  \textbf{rPGD-Diag}.
\end{itemize}
\begin{algorithm}[ht]
    \caption{\textbf{\ourAlgo} for matrix sensing loss $f(W, g)$}
    \label{alg:gwwt}
\begin{algorithmic}
    \STATE {\bfseries Input:} initialization $W_0$ and $g_0$, number of iterations $T$, step-sizes $\gamma_t$ and $\eta_t$.
    \FOR{$t=0, 1, 2, \cdots,T-1$}
        \STATE $V_t = W_t - \eta_t \nabla_{W}f(W_t, g_t)$
        \STATE $W_{t+1} = \frac{V_t}{\|V_t\|_F}$
        \STATE $g_{t+1} = g_t - \gamma_t \nabla_{g} f(W_t, g_t)$
    \ENDFOR
\end{algorithmic}
\end{algorithm} 
\begin{algorithm}[ht]
    \caption{\textbf{rPGD-Diag} for matrix sensing loss $f(W, D)$}
    \label{alg:wdwt}
\begin{algorithmic}
    \STATE {\bfseries Input:} initialization $W_0$ and $D_0$, number of iterations $T$, step-sizes $\gamma_t$ and $\eta_t$.
    \FOR{$t=0, 1, 2, \cdots,T-1$}
        \STATE $V_t = W_t - \eta_t \nabla_{W}f(W_t, D_t)$
        \STATE $W_{t+1} = V_t$ with all column vectors normalized.
        \STATE diag($D_{t+1}$) = diag($D_t$) $- \gamma_t$ diag($\nabla_{D} f(W_t, D_t)$)
    \ENDFOR
\end{algorithmic}
\end{algorithm}

\begin{figure}[tb]   
    \centering
    \begin{minipage}[t]{0.48\linewidth}
    \centering
    \includegraphics[width=\linewidth]{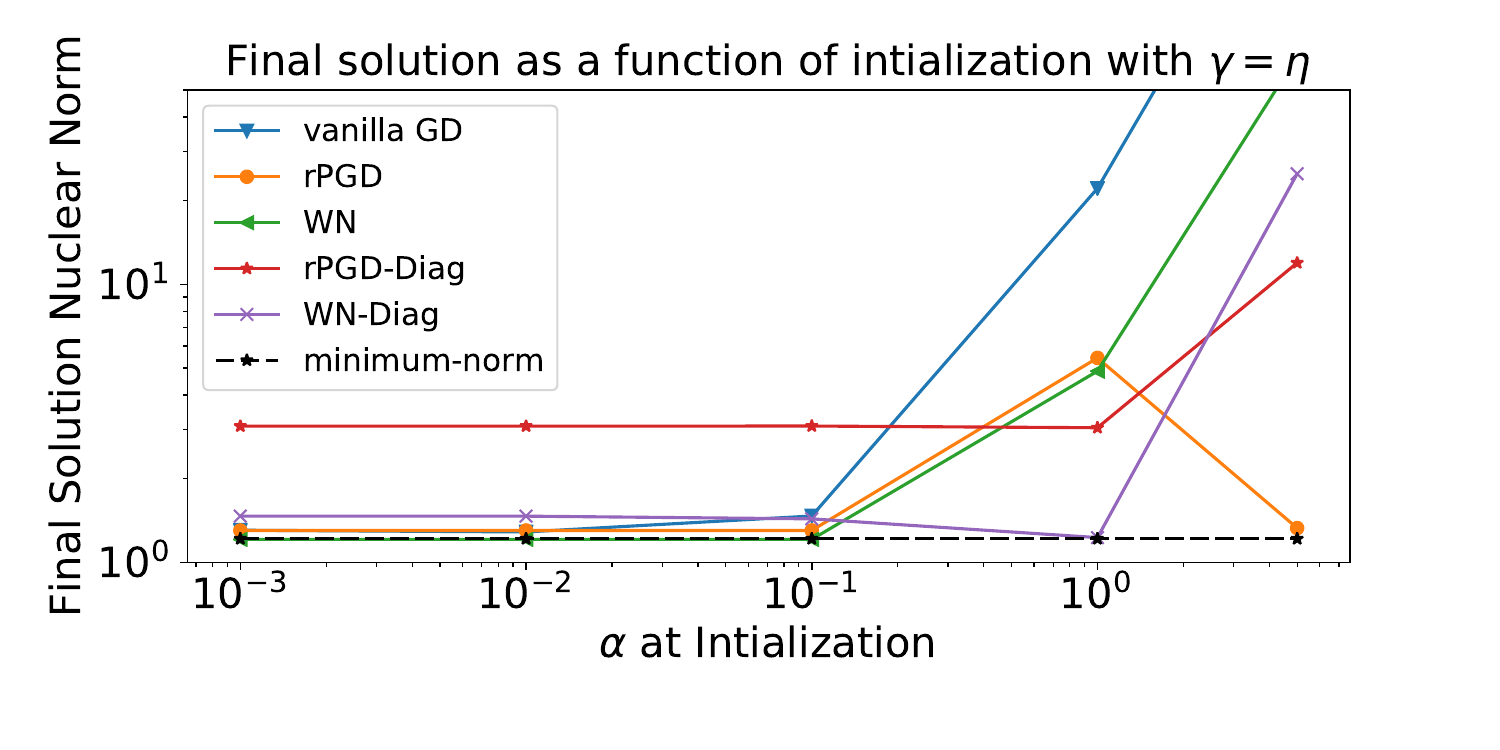}
    \end{minipage}
    \begin{minipage}[t]{0.48\linewidth}
    \centering
    \includegraphics[width=\linewidth]{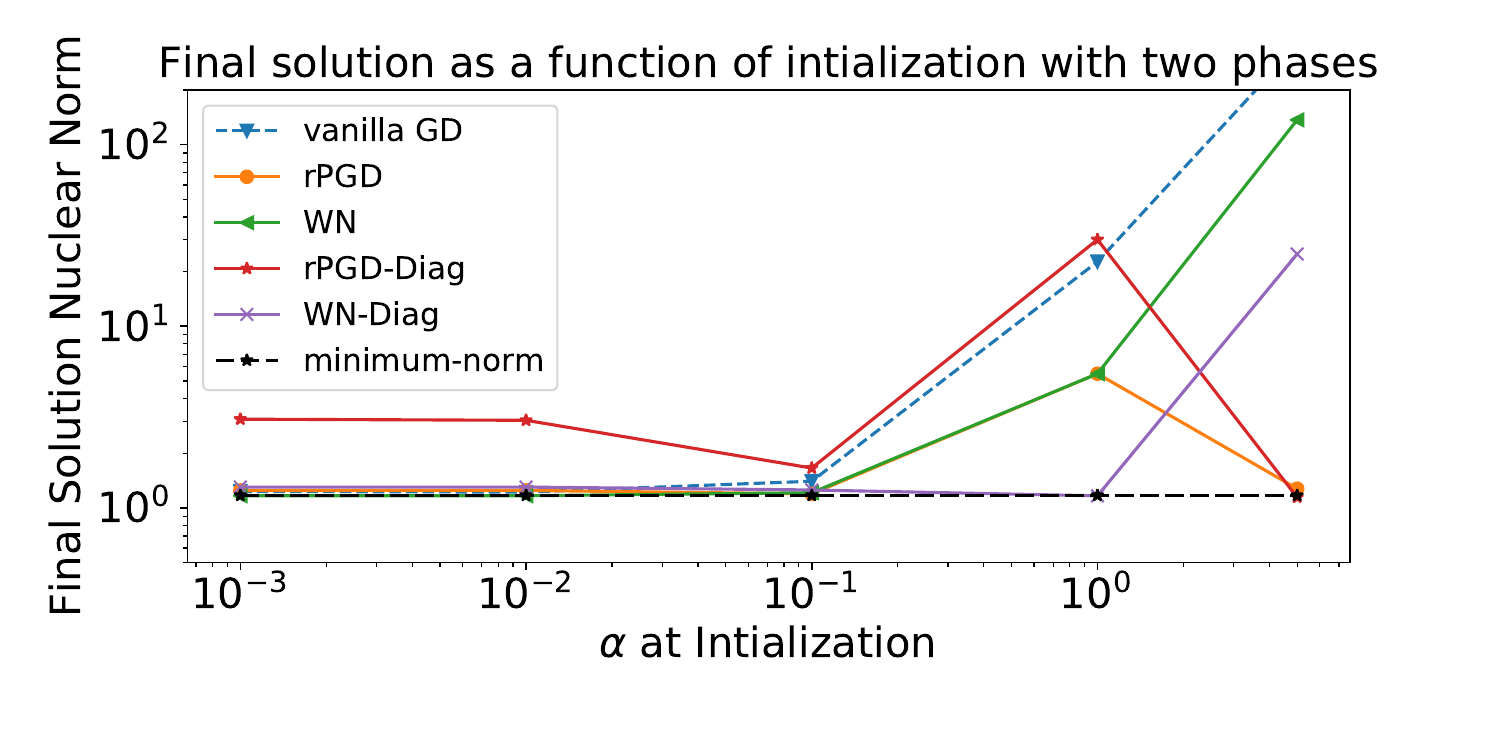}
    \end{minipage}
    \caption{Comparison of the final solutions $\|\widehat{X}\|_*$ provided by GD, WN, and \ourAlgo on an overparameterized matrix sensing problem (\eqref{eqn:uut}). All algorithms start from the same initialization with the scale $\alpha = \sqrt{\|UU^\top\|_F}$. Compared to GD, WN and \ourAlgo converge close to the minimum nuclear-norm solution for a broader region of initialization. The left plot is when we use the same stepsize for $W$ and $g$: $\gamma_t=\eta_t=c$. The right plot is when we use $\eta_t=c$ and $\gamma_t=c \mathbb{1}_{\{t>1000\}}$. 
    This suggests that Two-phase algorithm can arrive to a solution closer to the nuclear-norm solution for a broader range of $g_0$. The blue, green, and black curves of the top plot overlap when $0<\alpha<0.1$. The blue,  orange, green, purple, and black curves of the bottom plot overlap when $0<\alpha<0.1$.}  \label{fig:matrix}
\end{figure}

Denote the corresponding loss functions for rPGD as $f(W, g)$ and $f(W, D)$. Let $Z_0 = Z/\|Z\|_{F}$ where $Z$ is a matrix with i.i.d. Gaussian  entries, after which all column vectors have been normalized. We set the experiments with the following initialization:
\begin{itemize}
    \item For vanilla GD on $f(U)$, let $U_0=\alpha Z_0$;
    \item For WN and rPGD, let $W_0=Z_0$, and $g_0=\alpha^2$ ;\
    \item For WN-Diag and rPGD-Diag, let $W_0=Z_0$ and $D_0=\alpha^2I$. 
\end{itemize} 
We set $d=30$, $r=4$, and $m=60$. We simulate  $y_i=\langle{A_i,\hat{U}\hat{U}\rangle}$ with $\hat{U}\in\R^{d\times r}$ generated as a random matrix.\footnote{Code: $\hat{U}=\text{numpy.random.randn}(d,r)$. Note that this is not necessary the minimum  nuclear solution. We use the python package ``cvxpy" to solve for the minimum nuclear solution for \eqref{eqn:uut}.}


We compare the performance of gradient descent, and our algorithms for several initializations scales $g$. We run each algorithm until convergence (i.e., when the squared loss is less than $10^{-6}$). 

Similar to Figure 3, we use different learning rate schemes to get the final solution. We use grid search to find appropriate constant learning rate c.\footnote{Note $c$ varies for different $g_0$ and different algorithms. Here we start with $0.5$ and then decay by a factor of 2 until we get a step-size that can converge to the solution.} The top plot in Figure \ref{fig:matrix} uses the following learning rate: constant $c$ for gradient descent; $\eta_t=\gamma_t=c$ for rPGD (Algorithm~\ref{alg:gwwt} and \ref{alg:wdwt}); and set $\eta_t =\gamma_t=c\|W\|_F$  for WN.
 The bottom plot in Figure \ref{fig:matrix} uses the two phase learning rates: constant for gradient descent; $\eta_t=c$ and $\gamma_t=c \mathbf{1}_{\{t>1000\}}$ for rPGD (Algorithm~\ref{alg:gwwt} and \ref{alg:wdwt}); and set $\eta_t =\gamma_t=c\|W\|_F$ and $\gamma_t=c\mathbb{1}_{\{t>1000\}}$ for WN. Compared to GD, WN and \ourAlgo converge close to the minimum nuclear-norm solution for a larger region of initialization.  Moreover, these results also suggest that with the two-phase  algorithm, one can arrive to a solution  close to the nuclear-norm solution for a wider range of $g_0$.



\end{document}